\def\year{2018}\relax
\documentclass[letterpaper]{article} 
\usepackage{aaai18}  
\usepackage{times}  
\usepackage{helvet}  
\usepackage{courier}  
\usepackage{url}  
\usepackage{graphicx}  
\usepackage{multirow}
\mathchardef\mhyphen="2D 

\usepackage{epsfig, amssymb, amsfonts, bm, amsmath, ascmac, braket,amsthm,booktabs}
\usepackage{subfigure} 
\usepackage{amssymb}
\usepackage{mathtools,bbm}
\newtheorem{lemma}{Lemma}
\newtheorem{define}[lemma]{Definition}
\newtheorem{proposition}[lemma]{Proposition}
\newtheorem{theorem}[lemma]{Theorem}

\newtheorem{corollary}[lemma]{Corollary}

\newcommand{\argmin}{\mathop{\rm argmin}}

\usepackage{nicefrac}       
\usepackage{microtype}      

\begin{document}
%

\title{Hypergraph $p$-Laplacian: A Differential Geometry View}

\author{Shota Saito \\
The University of Tokyo\\
ssaito@sat.t.u-tokyo.ac.jp
\And
Danilo P Mandic\\
Imperial College London\\
d.mandic@imperial.ac.uk
\And
Hideyuki Suzuki\\
Osaka University\\
hideyuki@ist.osaka-u.ac.jp
}


\maketitle

\begin{abstract} 

The graph Laplacian plays key roles in information processing of relational data, and has analogies with the Laplacian in differential geometry. 
In this paper, we generalize the analogy between graph Laplacian and differential geometry to the hypergraph setting, and propose a novel hypergraph $p$-Laplacian. 
Unlike the existing two-node graph Laplacians, this generalization makes it possible to analyze hypergraphs, where the edges are allowed to connect any number of nodes. 
Moreover, we propose a semi-supervised learning method based on the proposed hypergraph $p$-Laplacian, and formalize them as the analogue to the Dirichlet problem, which often appears in physics. 
We further explore theoretical connections to normalized hypergraph cut on a hypergraph, and propose normalized cut corresponding to hypergraph $p$-Laplacian. The proposed $p$-Laplacian is shown to outperform standard hypergraph Laplacians in the experiment on a hypergraph semi-supervised learning and normalized cut setting.
\end{abstract}

\section{Introduction}
Graphs are a standard way to represent pairwise relationship data on both regular and irregular domains. One of the most important operators characterizing a graph is the graph Laplacian, which can be explained in several ways. For the example of spectral clustering~\cite{tuto}, we consider normalized graph cut~\cite{normcut,Yu03multiclasscuts}, random walks~\cite{randomwalkgraph,Gradyrandomwalk}, and analogues to differential geometry of graphs~\cite{Branin1966,Grady03,Zhou06,Bougleux07}.

Hypergraphs are a natural generalization of graphs, where the edges are allowed to connect more than two nodes~\cite{Berge84}. 
The data representation with a hypergraph is used in a variety of applications~\cite{huang2009video,liu2010robust,cellhypergraph,TanGCQBC14}. 
This natural generalization of graphs motivates us to consider a natural generalization of Laplacian to hypergraphs, which can be applied to hypergraph clustering problems. 
However, there is no straightforward approach to generalize the graph Laplacian to a hypergraph Laplacian. One way is to model a hypergraph as a tensor, for which we can define Laplacian~\cite{cooper2012spectra,hu2015laplacian} and construct hypergraph cut algorithms~\cite{bulo2009game,ghoshdastidar2014consistency}. 
However, this requires the hypergraph to obey a strict condition of a $k$-uniform hypergraph, where each edge connects exactly $k$ nodes. 
The second approach is to construct a weighted graph, which can deal with arbitrary hypergraphs. 
Rodriguez's approach defines Laplacian of arbitrary hypergraph as an adjacency matrix of weighted graph~\cite{Rod}. 
Zhou's approach defines a hypergraph from the normalized cut approach, and outperforms Rodriguez's Laplacian on a clustering problem~\cite{ZhouHyper}. 
However, Rodriguez's Laplacian does not consider how many nodes are connected by each edge, and Zhou's Laplacian is not consistent with the graph Laplacian. 
Although all of previous studies consider the analogue to graph Laplacian, none of them considers the analogue to the Laplacian from differential geometry. This allows us to further extend to more general hypergraph $p$-Laplacian, which is not extensively studied unlike in the case of graph $p$-Laplacian~\cite{pgraph,Zhou06}.

In this paper, we generalize the analogy between graph Laplacian and differential geometry to the hypergraph setting, and propose a novel hypergraph $p$-Laplacian, which is consistent with the graph Laplacian. We define gradient of the function over hypergraph, and induce the divergence and Laplacian as formulated in differential geometry. Taking advantage of this formulation, we extend our hypergraph Laplacian to a hypergraph $p$-Laplacian, which allows us to better capture hypergraph characteristics. We also propose a semi-supervised machine learning method based upon this $p$-Laplacian. Our experiment on hypergraph semi-supervised clustering problem shows that our hypergraph $p$-Laplacian outperforms the current hypergraph Laplacians.

The versatility of differential geometry allows us to introduce several rigorous interpretations of our hypergraph Laplacian. 
A normalized cut formulation is shown to yield the proposed hypergraph Laplacian in the same manner as in standard graphs. 
We further propose a normalized cut corresponding to our $p$-Laplacian, which shows better performance than the ones corresponding to current Laplacians in the experiments.
We also explore the physical interpretation of hypergraph Laplacian, by considering the analogue to the continuous $p$-Dirichlet problem, which is widely used in Physics.  {\it All proofs are in Appendix Section.}


\section{Differential Geometry on Hypergraphs}
\label{theory}

\subsection{Preliminary Definition of Hypergraph}
In this section, we review standard definitions and notations from hypergraph theory. We refer to the literature~\cite{Berge84} for a more comprehensive study. A {\it hypergraph} $G$ is a pair $(V,E)$, where 
$E \subset \cup_{k=1}^{|V|} \cup_{ \{v_1,\ldots,v_k\} \subset V} \{ [v_{\sigma(1)},\ldots,v_{\sigma(k)} ] \mid \sigma \in \mathcal{S}_{k}\}$, and $\mathcal{S}_{k}$ denotes the set of permutations $\sigma$ on $\{1,\ldots,k\}$. 
An element of $V$ is called a {\it vertex} or {\it node}, and an element of $E$ is referred to as an {\it edge} or {\it hyperedge} of the hypergraph. 
A hypergraph is {\it connected} if the intersection graph of the edges is connected. 
In what follows, we assume that the hypergraph $G$ is connected. 
A hypergraph is {\it undirected} when the set of edges are symmetric, and we denote a set of undirected edges as $ E_{un} = E / \mathcal{S} $, where $\mathcal{S} = \cup_{k=1}^{|V|} \mathcal{S}_k$. In other words, edges $[v_{1},v_{2},\ldots, v_{k}] \in E$ and $[v_{\sigma(1)},v_{\sigma(2)},\ldots, v_{\sigma(k)}] \in E$ are not distinguished in $E_{un}$ for any $\sigma \in \mathcal{S}_{k}$, where $k$ is the number of nodes in the edge. 
A hypergraph is {\it weighted} when it is associated with a function $w\colon E \rightarrow \mathbb{R}^{+}$. For an undirected hypergraph it holds that $w([v_{1},v_{2},\ldots, v_{k}]) $$=$$ w([v_{\sigma(1)},v_{\sigma(2)},\ldots, v_{\sigma(k)}])$. 
We define the {\it degree} of a node $v \in V$ as $d(v) = \sum_{e \in E: v \in e} w(e)$, while the degree of an edge $e \in E$ is defined as $\delta (e) = |e|$. To simplify the notation we write $\delta_{e}$ instead of $\delta(e)$. 

We define $\mathcal{H}(V)$ as a Hilbert space of real-valued functions endowed with the usual inner product
\begin{align}
\langle f,g \rangle_{\mathcal{H}(V)} \coloneqq \sum_{v \in V} f(v)g(v)
\end{align}
for all $f,g \in \mathcal{H}(V)$. Accordingly, the Hilbert space $\mathcal{H}(E)$ is defined with the inner product
\begin{align}
\langle f,g \rangle_{\mathcal{H}(E)} \coloneqq \sum_{ e \in E} \frac{1}{\delta_e!}f(e)g(e).
\end{align}
Note that $f$ and $g$ are defined for directed edges.

\subsection{Hypergraph Gradient and Divergence Operators} 
\label{graddiv}
We shall now extend standard graph gradient and divergence operators studied in~\cite{Zhou06} to hypergraphs, which can be considered as hypergraph analogues in both of discrete and continuous case. First, we propose to define hypergraph gradient as follows. 

\begin{define}
\label{defgrad}
 The hypergraph gradient is an operator $\nabla \colon \mathcal{H}(V) \rightarrow \mathcal{H}(E)$ defined by
\begin{align}
\label{grad}
&(\nabla \psi)([v_1, \ldots, v_{\delta_e}] = e )\coloneqq 
\frac{\sqrt{ w(e) } }{\sqrt{\delta_e - 1}}\sum_{i=1}^{\delta_e}\left(\frac{\psi(v_i)}{\sqrt{d(v_i)}} - \frac{\psi(v_1)}{\sqrt{d(v_1)}} \right)
\end{align}
\end{define}

The gradient is defined as a sum of a pairwise smoothness term between $e_{[1]}$ node and the others. 
Since the coefficient of graph gradient is defined as a square root of the weight $w(e)$~\cite{Zhou06}, 
we derive the coefficient for hypergraph by considering the average among the pairs between $e_{[1]}$ and the other node, $w(e)$ divided by $\delta_{e} - 1$, in order to normalize the effect of weight. 
For an undirected hypergraph, we define a gradient for an edge $e \in E_{un}$ and vertex $v$, i.e. $(\nabla \psi)(e,v)$.
Using the gradient defined for each edge, we can define the gradient at each node $v$ as $\nabla\psi(v) \coloneqq  \{(\nabla \psi)(e) / \delta_{e}! \mid e_{[1]} = v, e \in E\}$, where $e_{[1]}$ denotes the first element of edge $e$. 
Then, the norm of the gradient $\nabla \psi$ at node $v$ is defined by
\begin{align}
\| \nabla\psi(v) \|:=  \Bigl(  \sum_{ e \in E : e_{[1]} = v } \frac{(\nabla \psi)^2(e)}{\delta_{e}!}   \Bigr) ^{\frac{1}{2}}.
\end{align}
It then follows that the definition of this norm satisfies the conditions of a norm in a metric space. The {\it $p$-Dirichlet sum} of the function $\psi$ is given by
\begin{align}
\label{kdirichlet}
S_{p}(\psi) := \sum_{v \in V} \| \nabla\psi(v) \|^{p}.
\end{align}
Loosely speaking, the norm of the gradient on a node of a hypergraph measures local smoothness of the function around the node, and the Dirichlet sum measures total roughness over the hypergraph. 
Remark that $\| \nabla \psi\|$ is defined in the space $\mathcal{H}(E)$ as $\| \nabla \psi\| $$=$$ \langle \nabla \psi,\nabla \psi \rangle^{1/2}_{\mathcal{H}(E)}$, and satisfies $S_{2}$$ =$$ \| \nabla \psi  \|^2$.

\begin{define}
 The hypergraph divergence is an operator $\mathrm{div}: \mathcal{H}(E) \rightarrow \mathcal{H}(V) $ which satisfies $\forall \psi \in \mathcal{H}(V),\forall\phi \in \mathcal{H}(E)$
\begin{align}
\label{divstokes}
 \langle \nabla \psi, \phi \rangle_{\mathcal{H}(E)} = \langle \psi, -\mathrm{div} \phi\rangle_{\mathcal{H}(V)}.
\end{align}
\end{define}
Notice that Eq.~\eqref{divstokes} can be regarded as a hypergraph analogue of Stokes' Theorem on manifolds. 
The divergence can now be written in a closed form as follows:
\begin{proposition}
\label{propdiv}
\begin{align}
\nonumber
\mathrm{div} \phi (v) =& -\sum_{ e \in E : v \in e } \frac{\sqrt{w(e)} }{\delta_e!\sqrt{\delta_e-1}\sqrt{d(v)}} \phi(e)\\
\label{div}
&+\sum_{e \in E : e_{[1]} = v}\delta_e \frac{\sqrt{w(e)}}{\delta_e!\sqrt{\delta_e-1}\sqrt{d(v)}}\phi(e).
\end{align}
\end{proposition}
Intuitively, Eq.~\eqref{div} measures the net flows at the vertex $v$; the first term counts the outflows from the originator $v$ and the second term measures the inflow towards $v$. Note that this allows us to use Eq.~\eqref{div} as a definition of the divergence; it satisfies Eq.~\eqref{divstokes}, analogously to Stokes' theorem in the continuous case. Note also that divergence is always 0 if $\phi$ is undirected i.e $\phi(v_1,v_2,\dots,v_k)=\phi(v_{\sigma(1)},v_{\sigma(2)}, \dots,v_{\sigma(k)})$.
\subsection{Laplace Operators} 

In this section, we present the {\it hypergraph $p$-Laplace operator}, which can be considered as a discrete analogue of the Laplacian in the continuous case.
\begin{define}
The hypergraph Laplacian is an operator $\Delta_{p} \colon \mathcal{H}(V) \rightarrow \mathcal{H}(V)$  defined by
\begin{equation}
\label{pplaplacian}
\Delta_{p} \psi \coloneqq -\mathrm{div}(\| \nabla\psi \|^{p-2} \nabla \psi).
\end{equation}
\end{define}
This operator is linear for $p$$=$2. For an undirected hypergraph, we get the hypergraph $p$-Laplacian as follows; 
\begin{proposition}
\label{propplaplacian}
\begin{align}
\label{laplacian_written_down} 
(\Delta_{p}\psi)(v) =\sum_{u \in V\backslash\{v\}} \left( d_{p}(v) \frac{\psi(v)}{\sqrt{d(v)}} -w_{p}(u,v)\frac{\psi(u)}{\sqrt{d(u)}}  \right) 
\end{align}
denoting $w_{p}(v,v) = 0$ and
\begin{align}
\nonumber
&w_p(u,v) =
\sum_{e \in E_{un};u,v \in e} \frac{w(e)}{\delta_e - 1} \times\\
\nonumber &
\left(
-\|\nabla\psi_{e}\|^{p-2} + \|\nabla\psi(u)\|^{p-2} +  \|\nabla\psi(v)\|^{p-2}   
\right), \\
\nonumber
and\\
\nonumber
&d_p(v) = d(v) \|\nabla\psi(v)\|^{p-2} \\
\nonumber
&-
\sum_{e\in E_{un}; v \in e} 
\left(\frac{w(e)}{\delta_e - 1}
\left( 
- \|\nabla\psi_{e}\|^{p-2}  + \|\nabla\psi(v)\|^{p-2}
\right) 
\right),
\end{align}
where $\|\nabla\psi_{e}\|^{p} = \sum_{v' \in e} \| \nabla\psi(v') \|^{p}/\delta_{e}$. 
\end{proposition}
Let $W_{p}$ be a matrix whose elements $w_p(u,v)$, $D_{p}$ be a diagonal matrix whose elements $d(u,u) = \sum_{v \in V} d_p(u,v)$. 
For $p=2$ case, which is a standard setting for hypergraph Laplacian, it becomes 
\begin{align}
\label{laplacian_written_down_2} 
&(\Delta\psi)(v) =\sum_{u \in V\backslash\{v\}} \left( d(v) \frac{\psi(v)}{\sqrt{d(v)}} -w(u,v)\frac{\psi(u)}{\sqrt{d(u)}}  \right), \\ 
\nonumber
&\mathrm{where}\\
\nonumber
&w(u,v) = \sum_{e \in E_{un};u,v \in e} \frac{w(e)}{\delta_e - 1},\mathrm{\ } w(u,u) = 0.
\end{align}
We denote $W_{2}$ by $W$ and a diagonal matrix $D$ whose elements by $d(u,u) = d(u)$. Note that $d_p(u,u) = \sum_{u \in V} w_p(v,u)$. Using these matrices the Laplacian in~\eqref{pplaplacian} can be rewritten as
\begin{align}
\label{laplacian}
 (\Delta_{p}\psi) = D^{-1/2}(D_{p}-W_{p})D^{-1/2}\psi.
\end{align}
We shall denote the matrix associated with the Laplacian by $L_p =  D^{-1/2}(D_p-W_p)D^{-1/2}$, so that the Dirichlet sum can be rewritten by using $L_p$ as follows.
\begin{proposition}
\label{p-dirichletmatrix}
 The Dirichlet sum $S_p(\psi)$ can be rewritten as
\begin{align}
\label{dirichlethypergraph}
 S_p(\psi) = \psi^{\top}L_p\psi.
\end{align}
\end{proposition}
Note that $L_p$ depends on the function $\psi$, while $L \coloneqq L_2$ is independent. 
When the hypergraph degenerates into a standard graph and $p = 2$, $L$ coincides with the graph Laplacian. 

From the above analysis, the following three statements follow straightforwardly.
\begin{proposition}
\label{penergy}
$\langle \psi, \Delta_{p}\psi \rangle_{\mathcal{H}(V)} = S_{p}(\psi)$
\end{proposition}
\begin{corollary}
\label{semi-d}
 The Laplacian $L_p$ is positive semi-definite.
\end{corollary}
\begin{proposition}
\label{pdiff}
$ \frac{\partial}{\partial \psi}S_{p}(\psi) = p \Delta_{p} \psi$
\end{proposition}

{\bf Remark 1.} For the case of standard graph in this setting, the discussion in this section reduces to the discrete geometry for standard graphs, as introduced in~\cite{Zhou06}. This implies that our proposed definition is a natural generalization of discrete geometry for a graph.

\section{Hypergraph Regularization}
\subsection{Hypergraph Regularization Algorithm}
In this section, we consider the hypergraph regularization problem and propose a novel solution. Given the hypergraph $H = (V,E)$ and label set $\mathcal{Y} = \{-1,1\}$, and assume that the subset of $S \subset V$ is labeled, the problem is to classify the vertices in $V \backslash S$ using the label of $S$. 
To solve this, we formulate hypergraph regularization as follows. The regularization of a given function $y\in \mathcal{H}(V)$ aims to find a function $\psi^*$, which enforces smoothness on all the nodes of the hypergraph, and at the same time closeness to the values of a given function $y$, as follows:
\begin{align}
& \psi^{*} = 
\argmin_{\psi\in\mathcal{H}(V)}
\left(
S_p(\psi) + \mu \| \psi - y \|^2
\right),
\label{regopt}
\end{align}
where $y(v)$ takes $-1$ or 1 if $v$ is labeled, 0 otherwise. The first energy term represents the smoothness as explained in Eq.~\eqref{kdirichlet}, while the second term is a regularization term. Let $\mathcal{E}_p(\psi,y,\mu)$ be the objective function of Eq.~\eqref{regopt}. 
Since the positive power of positive convex function is also convex, $\mathcal{E}_p$ is a convex function for $\psi$, meaning that Eq.~\eqref{regopt} has a unique solution, satisfying 
\begin{align}
\label{diff}
\left. \frac{\partial \mathcal{E}_p}{\partial \psi} \right|_{v} = \frac{\partial}{\partial \psi}\| \nabla_{v} \psi \|^{p} + 2\mu(\psi(v) - y(v)) = 0,
\end{align}
for all $v \in V$. Using Proposition~\ref{pdiff} we can rewrite the left hand side of Eq.~\eqref{diff} as 
\begin{align}
\label{rewritediff}
p( \Delta_{p} \psi )(v) + 2\mu(\psi(v) - y(v)) = 0,\  \forall v \in V.
\end{align}
The solution to the problem~\eqref{regopt} is therefore also solution to~\eqref{rewritediff}. Substituting the expression of the Laplacian from Eq.~\eqref{laplacian_written_down} into Eq.~\eqref{rewritediff} yields

\begin{align}
\label{inducedopt}
p\sum_{u \in V}  l_p(u,v)\psi(u)
+ 2\mu (\psi(v) - y(v)) = 0,
\end{align}
where $l_p(u,v)$ is a element of $L_p$.
To solve Eq.~\eqref{inducedopt} numerically, we shall use the Gauss-Jacobi iterative algorithm, similarly to the discrete case introduced in~\cite{Bougleux07}. Let $\psi^{(t)}$ be the solution obtained at the iteration step $t$, the update rule of the corresponding linearized Gauss-Jacobi algorithm is then given by
\begin{align}
\label{updaterule}
&\psi^{(t+1)}(v) = \sum_{u \in V \backslash \{v\}} c^{(t)}(u,v) \psi^{(t)}(u) + m^{(t)}(v)y(v),\\
&\mathrm{where}
\nonumber\\
\nonumber
&c^{(t)}(u,v) =
- \frac{ p l_p^{(t)}(u,v) }{ p l_p^{(t)} (v,v) + 2 \mu}
\mathrm{and \ }
 m^{(t)}(v)  = \frac{ 2 \mu }{ p l_p^{(t)} (v,v) + 2 \mu },
\end{align}
and where $l_p^{(t)}(u,v)$ is a $p$-Laplacian defined by $\psi^{(t)}$. We take $\psi^{(0)} = y$ as an initial condition. 
The following theorem guarantees the convergence of the update rule for arbitrary $p$. 
\begin{theorem}
\allowdisplaybreaks[2]
\label{pconv}
The update rule Eq.~\eqref{updaterule} yields a convergent sequence.  
 Moreover, with notations $\alpha $$= 1/(1 + \mu)$ and $\beta$$= \mu /(1 + \mu)$, a closed form solution to Eq.~\eqref{regopt} for $p=2$ is
\begin{align}
 \label{analyticalclosedform}
 \psi = \beta(I - \alpha  D^{-1/2}WD^{-1/2})^{-1}y.
\end{align}
\end{theorem}

The update rule~\eqref{updaterule} can be intuitively thought of as an analogue to heat diffusion process, similar to the standard graph case~\cite{Zhou06}. At each step, every vertex is affected by its neighbors, which is normalized by the relationship among any number of nodes. At the same time, the neighbors also retains some fraction from their effects. The relative amount by which
these updates occur is specified by the coefficients defined in Eq.~\eqref{updaterule}.

\subsection{Physical Interpretation of Hypergraph Regularization}

In standard graph cases, regularization with the graph Laplacian can be explained as an analogue to the continuous Dirichlet problem~\cite{Grady03}, which is widely used in physics, particularly in fluid dynamics.
 To avoid confusion, the continuous calculus operators are referred to when $(c)$ is superscripted. The Dirichlet integral is defined as  
\begin{equation}
 S^{(c)}_p (\psi) = \int_{\Omega} \|\nabla^{(c)} \psi\|^{p} d \Omega,
\label{contiDirichlet}
\end{equation}
and is minimized when the Laplace equation
\begin{align}
\label{contiharmonic}
 \Delta^{(c)} (\psi) \coloneqq \mathrm{div}^{(c)}( \|\nabla^{(c)} \psi\|^{p-2}\nabla^{(c)}\psi) = 0
\end{align}
is satisfied~\cite{courant1962methods}. The parameter $p$ is a coefficient for characteristics of viscosity of fluid. The function $\phi$ satisfying the Laplace equation is called a harmonic function. Solving Eq.~\eqref{contiDirichlet} with a boundary condition makes it possible to find a plausible interpolation between the boundary points. 

From a physical standpoint, finding the shape of an elastic membrane is well approximated by the Dirichlet problem. One may think about a rubber sheet fixed along its boundary, and hung down by gravity. This setting can be written as Dirichlet problem, and the solution would give the most stable form of a rubber sheet, whose characteristics of elasticity is represented by $p$. To solve this numerically, we have to discretize this continuous function. With the pairwise effect between the nodes, solving the Dirichlet problem over a standard graph can be thought of as finding a plausible surface over the graph.

 In the graph setup, we can say that solving the Dirichlet problem over a standard graph corresponds to finding a plausible surface over the graph which favors boundary condition $y$. For the hypergraph setting, solving the hypergraph Dirichlet problem gives a plausible surface with the boundary $y$, and $p$ is a parameter for hypergraph; this is achieved by considering not only the pairwise effects, but also the interactions among any number of nodes. In fact, if we discretize the continuous domain of definition into lattice and consider the effect of the next neighbor, the second-order differential operator is given by $D-W$ when $p=2$. Interestingly, if we set up the lattice as a hypergraph, which means that we take into account any number of neighbors at the same time, the second-order differential operator is also $D-W$.

\section{Hypergraph Cut}

\subsection{Revisiting the Hypergraph Two-class Cut} 
From the discussion so far, it is to be expected that there exists a relationship between hypergraph spectral theory and the considered manifold setup. Similarly to the case of standard graph and Zhou's hypergraph Laplacian, we now introduce the hypergraph cut problem that has a connection to our Laplacian, whereby a hypergraph can be partitioned into two disjoint sets, $A$ and $B$, $A \cup B = V$, and $A\cap B = \emptyset$. The normalized hypergraph cut can now be formulated as a minimization problem given by
\begin{align}
\label{ncutNP}
 &Ncut(A,B) =  \partial V(A,B)\left(\frac{1}{\mathrm{vol}(A)} + \frac{1}{\mathrm{vol}(B)}\right),\\
\nonumber
&\mathrm{where}\\
\label{cuthypergraph}
& \partial V(A,B) \coloneqq \sum_{u\in A,v\in B} \sum_{e \in E_{un} : u,v \in e}\frac{w(e)}{\delta_{e}-1} = \sum_{u \in A, v \in B} w(u,v),
\end{align}
and $\mathrm{vol}(A) = \sum_{u\in A}d(u).$ Note that this setting is consistent with the normalized cut problem on a standard graph.
Let $f \in \mathcal{H}(V)$, be a $|V|$ dimensional indicator vector function; $f(v)= a$ if node $v$ is in $A$, $f(v) = -b$ otherwise, where $a = 1/\mathrm{vol}(A)$ and $ b = 1/\mathrm{vol}(B)$. 
With these notations the problem~\eqref{ncutNP} can be rewritten as Rayleigh quotient:
\begin{align}
\nonumber
 \mathrm{min \ } Ncut(A,B) = \frac{f^{\top} D^{-\frac{1}{2}} (D-W) D^{-\frac{1}{2}} f}{f^{\top}f} \nonumber \\
 \label{rayliegh_2nd}
 \mathrm{\ s.t.}  \hspace{0.5em} \sqrt{d(v)}f(v) \in \{a,-b\}, f^{\top}D\mathbf{1} = 0, 
\end{align}
 Minimizing $Ncut$ is NP-hard, but it can be relaxed if we embed this problem in the real domain, and the solution of the relaxed problem is given by the second smallest eigenvalue of $L$~\cite{matcomp,normcut}.  

This setting is somewhat different from the work by~\citeauthor{ZhouHyper}~\shortcite{ZhouHyper} : if we replace the denominator of Eq.~\eqref{cuthypergraph} from $(\delta_e -1)$ to $\delta_e$, then it is exactly same as~\cite{ZhouHyper}. This difference from Zhou's approach allows for the proposed setting to be consistent with standard graphs and standard random walk setting whereas Zhou's setting can be seen as a case of the lazy random walk, as discussed in Appendix.

\subsection{Hypergraph Multiclass Cut}

We shall now extend two-class cut to multiclass cuts and establish the connection between this setting and our proposed Laplacian, similarly to~\cite{ZhouHyper}. In the standard graph case, multiclass clustering problem corresponds to decomposing $V$ into $k$ disjoint sets; $V = \cup_{i=1}^{k} V_i$ and $V_i \cap V_j = \emptyset $ for $ i \neq j$. We shall denote this multiclass clustering by $\Gamma_{V}^{k} = \{V_1,\ldots, V_k\}$, and formulate this problem as that of minimizing 
\begin{align}
\label{kncuts}
 Ncut(\Gamma_{V}^{k}) = \sum_{i=1}^{k} \frac{\partial V(V_i,V \backslash V_i)}{\mathrm{vol}(V_i)},
\end{align}
where $\partial V(V_i,V \backslash V_i) / \mathrm{vol}(V_i)$ measures the total weights of the links from $V_i$ to other clusters in $G$. We denote the multiclass clustering $\Gamma_{V}^{k}$ by a $|V| \times k$ matrix $X$, where $x(u,i) = 1$ if node $u$ belongs to the $i$th cluster and 0 otherwise. This allows us to rewrite the problem as 
\begin{align}
\label{minkncut}
 \nonumber
 \mathrm{min. \ } Ncut(\Gamma_{V}^k) &= \sum_{i=1}^{k}\frac{X_{i}^\top(D - W) X_i}{X_{i}^{\top}DX_i} \\
 &\mathrm{\ s.t.}  \hspace{0.5em} X \in \{1,0\}^{N \times k},X \mathbf{1}_{k} = \mathbf{1}_{|V|}.
\end{align}
To this end, we consider relaxing the constraints $X$ by minimizing 
$Ncut(\Gamma_{V}^{k}) $
with constraints $\tilde{Z}^{\top}\tilde{Z}=I_k$ where $\tilde{Z}=D^{1/2}Z$ and $Z=X(X^{\top}DX)^{(-1/2)}$. The optimal solution to this problem is given by the eigenvectors associated with the smallest $k$ eigenvalues of the Laplacian $L$. Similarly to Zhou's Laplacian, the following proposition holds.
\begin{proposition}
\label{multiprop}
 Denote the eigenvalues of Laplacian $L$ by $\lambda_{1} \leq \cdots \leq \lambda_{|V|}$, and define $c_{k}(H) = \min Ncut \mathrm{\ s.\ t.\ }X \in \{1,0\}^{N \times k},X \mathbf{1}_{k} = \mathbf{1}_{|V|} $. Then $\sum_{i=1}^k \lambda_{i} \le c_k(H)$.
\end{proposition}
As discussed in~\cite{ZhouHyper}, this result shows that the result of the real-value relaxed optimization problem gives us a lower bound for the original combinatorial optimization problem. However, it is not clear how to use the $k$ eigenvectors to obtain $k$ clusters. For a standard graph, applying the $k$-means method to the $k$ eigenvectors heuristically performs well, and this approach can be applied to the hypergraph problem as well.

\subsection{Hypergraph $p$-Normalized Cut}
From the above discussion, one might expect that there exists corresponding hypergraph cut induced from hypergraph $p$-Laplacian, similarly to the graph $p$-Laplacian case~\cite{pgraph}. Since $p$-Laplace operator is nonlinear, we need to define eigenvalues and eigenvectors.
\begin{define}
Hypergraph $p$-eigenvalue $\lambda_p \in \mathbb{R}$ and $p$-eigenvector $\psi \in \mathcal{H}(V)$ of $\Delta_p$ are defined by
\begin{align}
\label{eigenvec}
(\Delta_{p}\psi)(v) = \lambda_p \xi_{p}(\psi(v)),\mathrm{where}\xi_p(x) = |x|^{p-1} \mathrm{sgn}(x).
\end{align}
\end{define}
To obtain $p$-eigenvector and $p$-eigenvalue, we consider Rayleigh quotient and the following statements follow:
\begin{proposition}
\label{critical_eigen}
Consider the Rayleigh quotient for $p$-Laplacian,
\begin{align}
\label{rayliegh}
R_{p}(\psi) = \frac{S_{p}(\psi)}{\|\psi\|_{p}^{p}}, \mathrm{where\ } \|\psi\|_{p} = (\sum_{v} \psi^p(v))^{1/p}.
\end{align}
The function $R_p$ has a critical point at $\psi$ if and only if $\psi$ is $p$-eigenvector of $\Delta_{p}$. 
The corresponding $p$-eigenvalue $\lambda_{p}$ is given as $\lambda_p$$=$$R_p (\psi)$. 
Moreover, we have $R_p (\alpha \psi)$$ =$$ R_p(\psi),$ $\forall \psi \in H(V)$ and $\alpha \in \mathbb{R}, \alpha \neq 0$. 
\end{proposition}
\begin{corollary}
\label{1steigen}
The smallest $p$-eigenvalue $\lambda_{p}^{(1)}$ equals to 0, and corresponding $p$-eigenvector is $D^{1/2}\mathbf{1}.$
\end{corollary}
Eq.~\eqref{rayliegh} is analogue to the continuous nonlinear Rayleigh quotient
\begin{align}
R_{p}^{(c)}(\psi) = \frac{ \int_{\Omega} \|\nabla^{(c)} \psi\|^{p} d \Omega  }{ \int_{\Omega} \| \psi\|^{p}  d \Omega },
\end{align}
which relates to nonlinear eigenproblem.

In order to define a hypergraph cut corresponding to hypergraph $p$-Laplacian, let us consider $f \in \mathcal{H}(V)$, a $|V|$ dimensional indicator vector function in Eq.~\eqref{rayliegh_2nd}. Then substituting $f$ into Eq.~\eqref{rayliegh} gives
\begin{align}
\nonumber
 \mathrm{min \ } p \mhyphen Ncut(A,B) \coloneqq \frac{f^{\top} D^{-\frac{1}{2}} (D_p-W_p) D^{-\frac{1}{2}} f}{\|f \|_{p}^{p}} \nonumber \\
 \label{prayliegh}
 \mathrm{\ s.t.}  \hspace{0.5em} \sqrt{d(v)}f(v) \in \{a,-b\}, f^{\top}D\mathbf{1} = 0, 
\end{align}
which can be seen as the cut corresponding to our hypergraph $p$-Laplacian. 
The problem~\eqref{prayliegh} is NP-hard, and therefore we need to consider a relaxed problem, similarly to the case of $p$$=$$2$. 
The constraints in Eq.~\eqref{prayliegh} require the second eigenvector to be orthogonal to the first eigenvector. However, the orthogonal constraint is not suitable for $p$-eigenvalue problem, since the $p$-Laplacian is nonlinear and therefore eigenvectors are not necessary to be orthogonal to each other.

For $p=2$ case, since we see
\begin{align}
\label{var}
\| \psi \|_2^2 = \| \psi -  \frac{ \langle \psi,D^{\frac{1}{2}}\mathbf{1} \rangle }{|V|}  D^{\frac{1}{2}}\mathbf{1}\| = \min_{c \in \mathbb{R}} \|\psi - cD^{\frac{1}{2}}\mathbf{1} \|
\end{align}
by $\langle \psi, D^{\frac{1}{2}}\mathbf{1} \rangle$$ =$$0$ for the second eigenvector, the Rayleigh quotient to get the second eigenvector $v^{(2)}$ can be written as 
\begin{align}
 v^{(2)} = \argmin_{\psi \in H(V)}\frac{S_{p}(\psi)}{\min_c\|\psi - c D^{\frac{1}{2}}\mathbf{1}\|^2}.
\end{align}
Motivated by this, we here define the Rayleigh quotient for the second smallest $p$-eigenvalue as
\begin{align}
\label{rayliegh_2}
R^{(2)}_{p}(\psi) = \frac{S_{p}(\psi)}{\min\|\psi - c D^{\frac{1}{2}}\mathbf{1}\|_{p}^{p}},
\end{align}
This quotient is supported by the following theorem.
\begin{theorem}
\label{psecondlaplacian}
 The global minimum of Eq.~\eqref{rayliegh_2} is equal to the second smallest $p$-eigenvalue $\lambda_p^{(2)}$ of $\Delta_p$. 
The corresponding $p$-eigenvector $\psi^{(2)}_p$ can be obtained by $\psi^{(2)}_p = \psi^{*} - c^{*}D^{\frac{1}{2}} \mathbf{1}$, for any global minimizer $\psi^{*}$ of $R^{(2)}$, where $c^{*} = \argmin_{c \in \mathbb{R}} | \sum \psi^{*}(v) - \sqrt{d(v)}c |^p$.
Moreover, we have $R^{(2)}_p (t\psi + c) = R_p^{(2)} (\psi)$ where $t \in \mathbb{R}, t \neq 0,$ and $ \forall c \in \mathbb{R}$.
\end{theorem}
Therefore, for the relaxed problem of Eq.\eqref{prayliegh}, we consider the Rayleigh quotient Eq.\eqref{rayliegh_2}. 
We note that, as $R^{(2)}_p$ is not convex, optimization algorithms might have danger not to achieve the global minimum.
However, since the function $R_p^{(2)}$ is continuous for $p$, we can assume that the global minimizer of $R_{p_1}^{(2)}$ and of $R_{p_2}^{(2)}$ are close, if $p_1$ and $p_2$ are close. Hence, we firstly obtain the second eigenvector for $p=2$, where there exist more stable algorithms to obtain eigenvectors, and use it as the initial condition for optimization algorithms for $p \neq 2$.

\subsection{Comparison to Existing Hypergraph Laplacians and Related Regularizer} 

We now compare our Laplacian with other two standard ones. 
~\citeauthor{ZhouHyper}~\shortcite{ZhouHyper} have proposed the Laplacian $L_{Z} $$=$$ I $$-$$ D_{v}^{-\frac{1}{2}}HW_{e}D_{e}^{-1}H^{\top}D_{v}^{-\frac{1}{2}}$ based on a normalized cut and lazy random walk view, where the degree matrices $D_{v}$ and $D_{e}$ stand respectively for diagonal matrices, containing degree of nodes and edges, 
$W_e$ is a diagonal matrix containing the weights of edges, and indices matrix $H$ is a $|V|$$\times$$ |E_{un}|$ matrix whose element $h(v,e)$$=$$1$ if node $v$ is connected to the edge $e$, and 0 otherwise. 
In this setting the hypergraph is represented by the matrix $HW_{e}D_{e}^{-1}H^{\top}$, where weights $W_e$ is normalized by degree of edges $D_e$.  
This Laplacian gives the same Laplacian if we consider the standard graph, except for the coefficient 1/2. 
This difference comes from the consistency of a lazy random walk view as explained in Appendix. 
\citeauthor{Agarwal06}~\shortcite{Agarwal06} shows that Zhou's Laplacian is equivalent to hypergraph star expansion in~\cite{Zien99} and~\cite{Li1996}. 

Another Laplacian has been proposed under the unweighted setting by~\citeauthor{Rod}~\shortcite{Rod} and is referred to as Simple Graph Method in~\cite{ZhouHyper}. 
The hypergraph is represented by a matrix $H W_{e}H^{\top} $$-$$ D_{v}$ and Laplacian is defined as $L_{R}$$ =$$ I $$-$$ D_{R}^{-1/2}H W_{e}H^{\top}D_{R}^{-1/2}$, where $D_{R}$ is a diagonal matrix whose elements are $d_{R}(u,u)$$ =$$ \sum_{v \in V} w_{R} (u,v)$ and $w_{R}(u,v) $$=$$ \sum_{e \in E_{un}:u,v \in e } w(e)$. 
This view is consistent with the standard graph, but it does not consider the difference of edge degree $\delta_e$. 
Rodriguez Laplacian is theoretically equivalent to hypergraph clique expansion in~\cite{Zien99},~\cite{Bolla93}, and~\cite{Gibson2000} as shown in~\cite{Agarwal06}.

Our Laplacian can be regarded as a family of Rodriguez's Laplacian, but we normalize the weight by the edge degree $\delta_e -1$ when constructing Laplacian, whose interpretation is in the definition of gradient. 
If we consider the clique constructed by $w(e)$, and also to normalize $w(e)$ by $\delta_e $$-$$1$, we obtain our 2-Laplacian.
Moreover, from the viewpoint of differential geometry, we obtain Rodriguez's Laplacian by changing the denominator in definition of gradient (Def.~\ref{defgrad}) from $\sqrt{\delta_{e} -1}$ to $1$. 

Note that Zhou's, Rodriguez's, and our Laplacian can be seen to reduce a hypergraph to an ordinary graph, whose adjacency matrix $HW_{e}D_{e}^{-1}H^{\top}$, $H W_{e}H^{\top} - D_{v}$, and $W$ respectively. 
Moreover, Rodriguez's and our Laplacian can be constructed from the graph gradient in~\cite{Zhou06} using the graph reduced from hypergraph. 
However, because our hypergraph gradient is different from graph gradient, we cannot construct our hypergraph $p$-Laplacian from graph $p$-Laplacian in~\cite{Zhou06} or another definition of graph $p$-Laplacian in~\cite{pgraph}.
For example, 
consider an undirected hypergraph $G$ where $V$$=$$\{v_1, v_2, v_3 \}$, $E$$ =$$ \{e = \{v_1, v_2, v_3 \}\}$ and $w(e)$$ =$$ 1$, $p$$=$$1$ and the function $\psi(v_1)$$=1, \psi(v_2) $$= $$0,$ and $\psi(v_3) $$= $$0$. 
In this setting we get $\Delta_1(v_1)$$ =$$ 4/\sqrt{6}$, and if we consider a graph reduced from hypergraph Zhou's for $v_1$ is $1/2$$ +$$ 1/2\sqrt{2}$ and unnormalized and normalized B\"uhler's for $v_1$ are both 1/2, while our and Zhou's 2-Laplacians give the same values. Details are in Appendix.


%

\citeauthor{TotalVariation}~\shortcite{TotalVariation} proposed a semi-supervised clustering using a $p$-regularizer $S_{p}^{(h)}=$ $\sum_{e} w(e)$ $(\max_{v\in V}$$ \psi(v) - \min_{u \in V} \psi (u))^{p}$, induced from a total variation on hypergraph, which is Lov\'asz expansion of a hypergraph cut in Eq.~\eqref{cuthypergraph}. 
Moreover, they use total variation $S_1^{(h)}$ for hypergraph cut, which favors balance while ours and the others favor to be attracted by larger hypergraph weights. 
This regularizer is reduced to the same one composed from graph $p$-Laplacian in~\cite{pgraph,nodal} when we consider a standard graph.  

\section{Experiments}

\begin{table*}[!t]
\begin{center}
\caption{Dataset summary. All datasets were taken from UCI Machine Learning Repository.}
\label{tab:sum}
\begin{tabular}{c|ccccccc}
\hline
 & mushroom & cancer & chess & congress & zoo & 20 newsgroups & nursery\\
\hline
\# of classes &  2 & 2 & 2 & 2 & 7 & 4 & 5\\
$|V|$ &  8124 &  699 & 3196 & 435 & 101 & 16242 & 12960\\
$|E|$ & 112 & 90  & 73 & 48 & 42 & 100 & 27\\
$\sum_{e \in E} |e|$ & 170604 & 6291 & 115056 & 6960 & 1717 & 65451 & 103680\\
\hline
\end{tabular}
\end{center}
\end{table*}

We compare the proposed hypergraph Laplacian with other hypergraph Laplacians, Zhou's and Rodriguez's, and Hein's regularizer, on categorical data, where for each instance one or more attributes are given. 
Each attribute has small number of attribute values, corresponding to a specific category. 
We summarize the benchmark datasets we used in Table~\ref{tab:sum}. 
As in~\cite{ZhouHyper}, we constructed a hypergraph for each dataset, where each category is represented as one hyperedge ,and set the weight of all the edges as 1. 
We used semi-supervised learning and clustering to classify the objects of a dataset. 
For fair comparison, we evaluate the performance with error rate, which is used in the previous studies on hypergraph clustering~\cite{ZhouHyper,TotalVariation}.

{\bf Semi-supervised Learning.} 
We compared our semi-supervised learning method, shown in Eq.~\eqref{updaterule} with the existing ones using the Laplacians reported by~\citeauthor{ZhouHyper}~\shortcite{ZhouHyper} and Rodriguez~\shortcite{Rod}. 
There are variety of ways to extend two-class clustering to multiclass clustering~\cite{bishop2007}, but in order to keep the comparison simple, we conducted the experiment only on two-class datasets. 
The parameter $\mu$ was chosen for all methods from $10^{k}$, where $k$$\in$$\{0,1,2,3,4\}$ by 5-fold cross validation. 
We randomly picked up a certain number of labels as known labels, and predicted the remaining ones. 
We repeated this procedure 10 times for different number of known labels. 
For our $p$-Laplacian, we varied $p$ from 1 to 3 with the interval of 0.1, and we show the result of $p$$=$$2$ and the result of $p$ giving the smallest average error for each number of known labeled points. 
The parameter $p$ for Hein's regularizer is fixed at 2, since this is recommended by~\citeauthor{TotalVariation}~\shortcite{TotalVariation}.
The results are shown in Fig.~\ref{exp:semi}~(a)-(d). 
When $p$$=$$2$ our Laplacian almost consistently outperformed Rodriguez's Laplacian, and showed almost the same behavior as Zhou's.
This means that normalizing hypergraph weights by the edge degree when constructing Laplacian enhances the performance. 
When we tuned $p$, our Laplacian consistently outperformed other Laplacians, and Hein's regularizer except the mushroom. 
The dataset mushroom might fit to the Hein's balanced cut assumption more than the other Laplacian's normalized cut assumption. 
Table~\ref{tab:ps} shows the values of $p$ which give the first and the second smallest average error. 
We can observe that $p$ giving the optimal error to each number of known label points would give close value to other points. 
This result implies that $p$ is a parameter for each hypergraph, rather than a parameter for the number of known label points. 
This can be seen as the analogue to fluid dynamics, where $p$ is a coefficient for characteristics of viscosity of each fluid.
\begin{figure*}[!t]
\begin{center}
\subfigure[Mushroom]{%
\includegraphics[width=.21\hsize,clip]{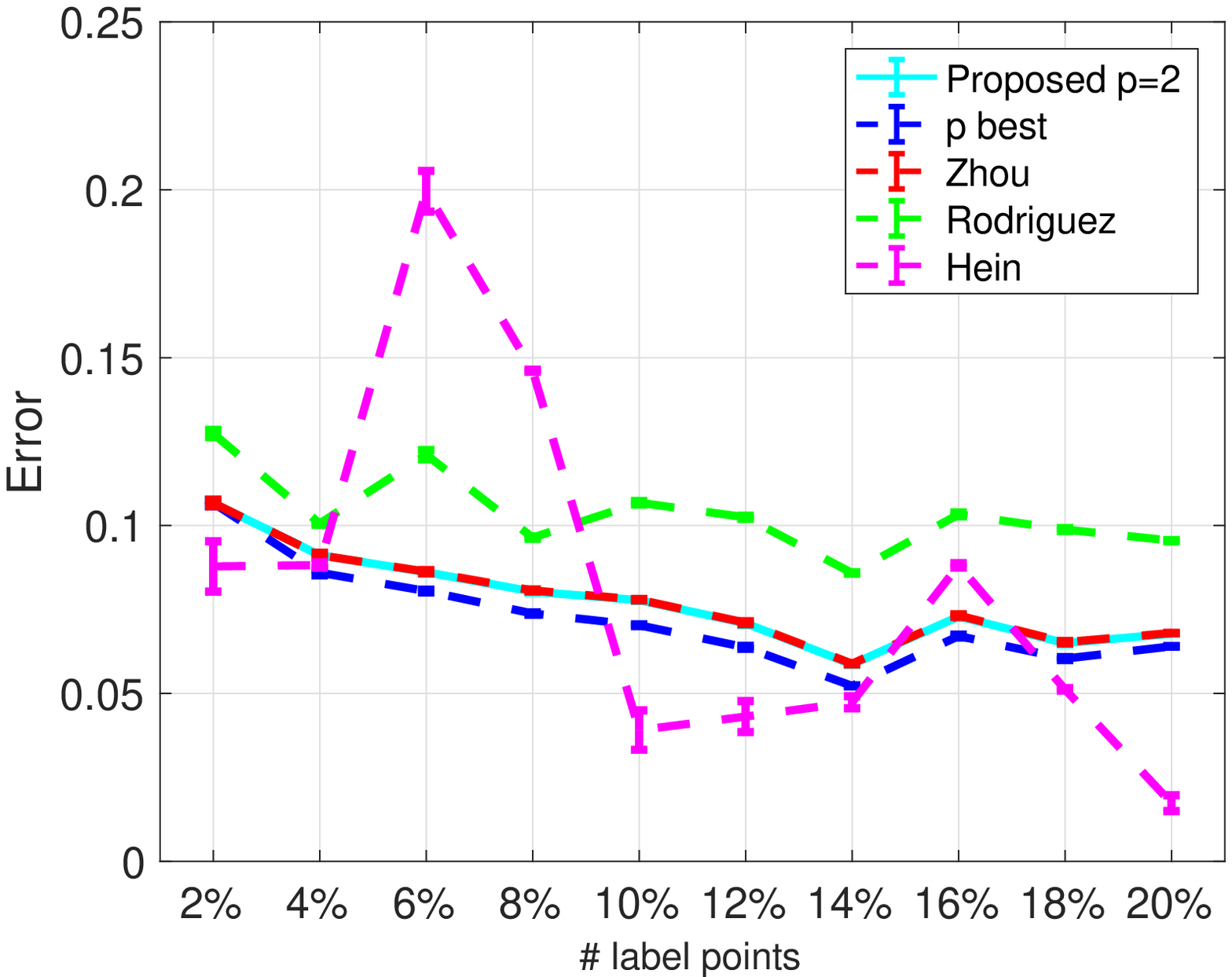}\label{fig:mushroom}}
\subfigure[Breast Cancer]{%
\includegraphics[width=.21\hsize,clip]{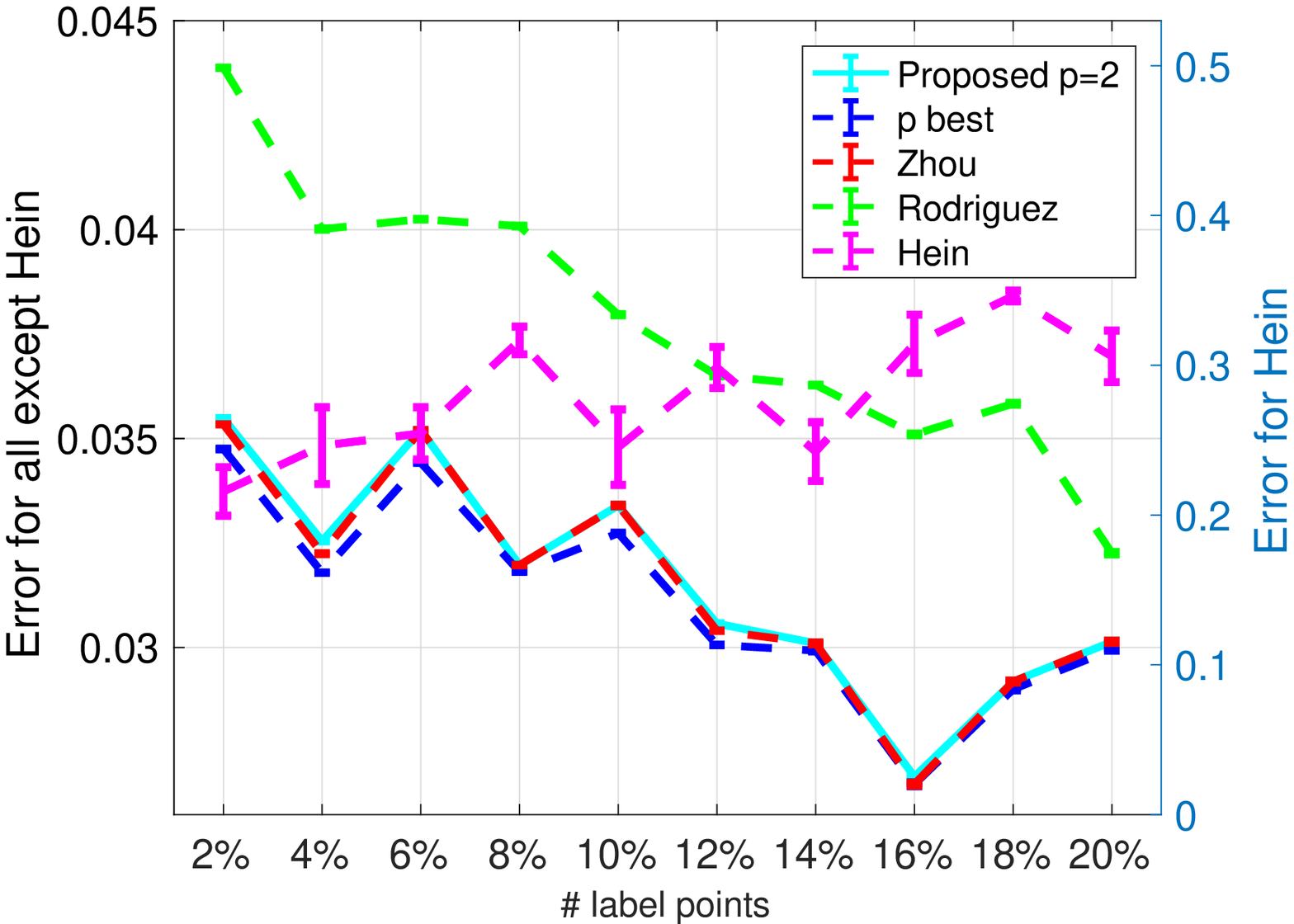}\label{fig:canser}}
~\subfigure[Chess]{%
\includegraphics[width=.21\hsize,clip]{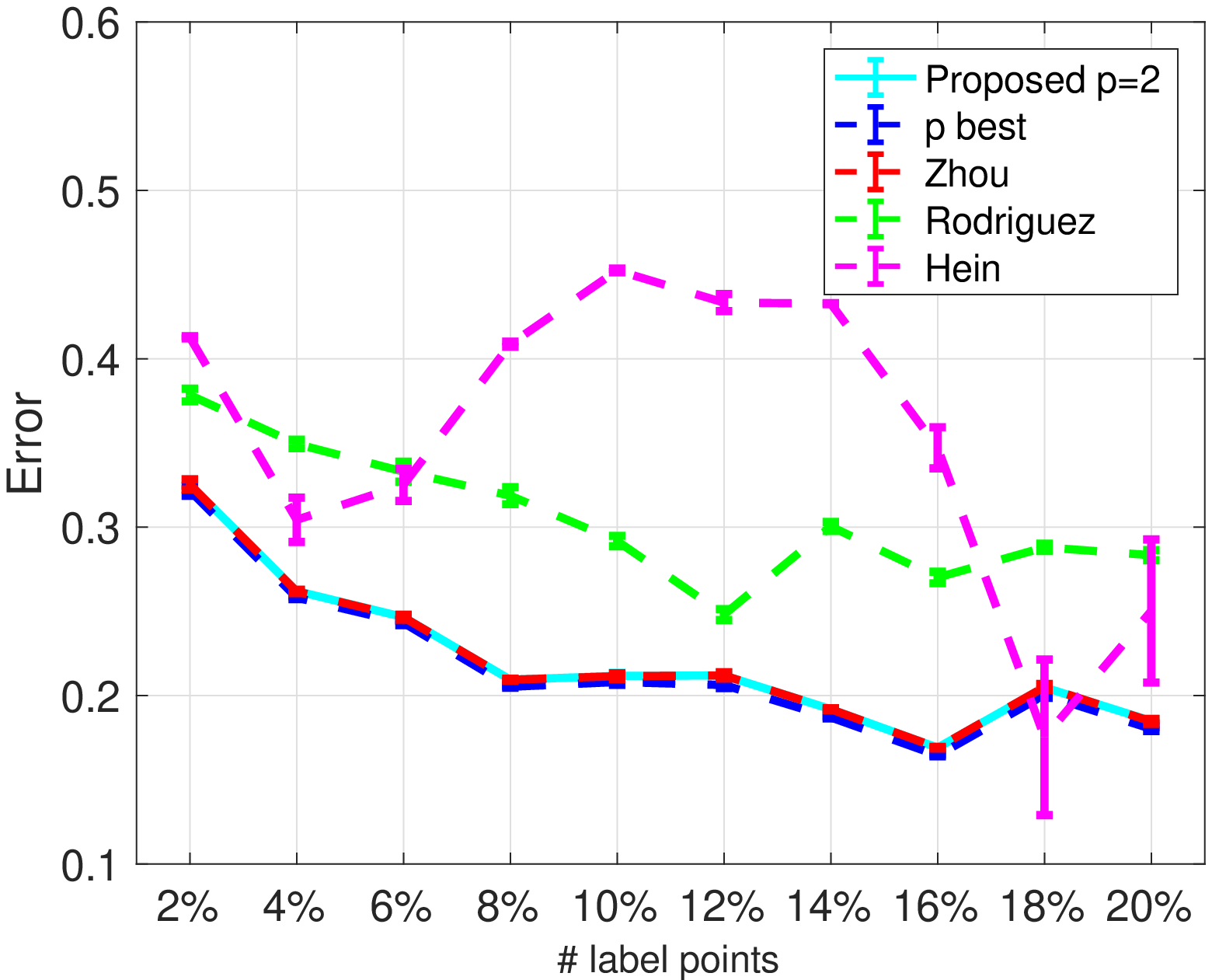}\label{fig:chess}}
 ~\subfigure[Congress]{%
\includegraphics[width=.21\hsize,clip]{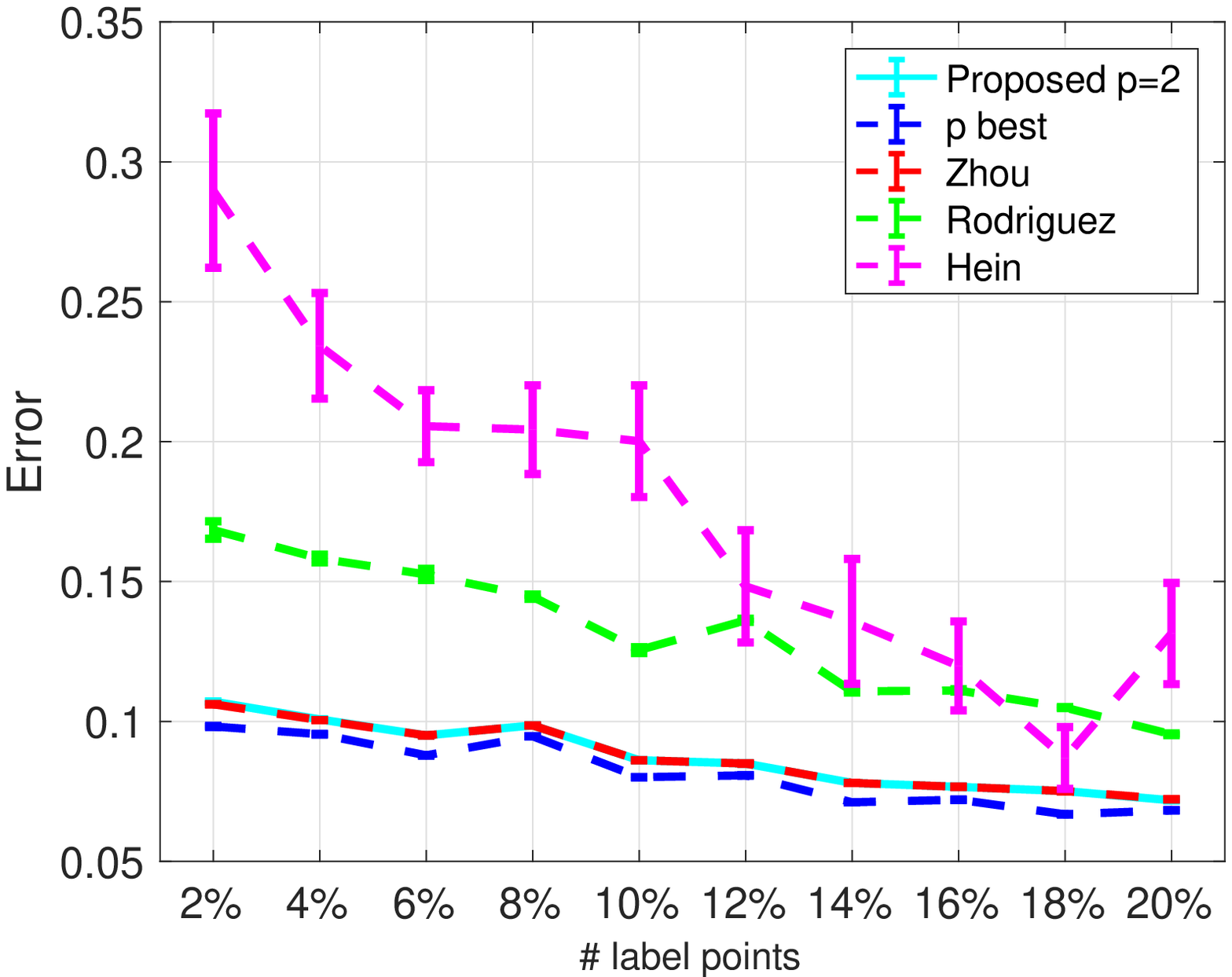}\label{fig:congress}}
\caption{Results of semi-supervised learning from proposed method and the state-of-the-art}
\label{exp:semi}
\end{center}
\end{figure*}
\begin{table*}[!t]
\begin{center}
\caption{The value of $p$ which gives an optimal error.}
\label{tab:ps}
\begin{tabular}{c|c|cccccccccc}
\hline
Dataset & error & \multicolumn{10}{c}{The portion of \# of known label points} \\
\hline
& & 2\% & 4\% & 6\% & 8\% & 10\% & 12\% & 14\% & 16\% & 18\% & 20\% \\
\hline
\multirow{2}{*}{Mushroom} & smallest & 1.6 & 1.6 & 2.3 & 1.5 & 1.3 & 1.3 & 1.7 & 1.7 & 1.0 & 1.4 \\
& 2nd smallest & 2.0 & 2.1 & 2.4 & 1.6 & 1.4 & 1.4 & 1.6 & 1.4 & 1.1 & 1.7 \\
\hline
\multirow{2}{*}{Breast Cancer}& smallest & 2.7 & 2.4 & 2.7 & 2.7 & 2.7 & 2.7 & 2.5 & 2.5 & 2.6 & 2.6 \\
&2nd smallest&2.6 & 2.1 & 2.4 & 2.6 & 2.6 & 2.6 & 2.6 & 2.6 & 2.7 & 2.7 \\
\hline
\multirow{2}{*}{Chess}&smallest&1.2 & 1.0 & 1.0 & 1.0 & 1.0 & 1.1 & 1.0 & 1.9 & 1.0 & 1.0 \\
&2nd smallest&1.1 & 1.1 & 1.1 & 1.5 & 1.1 & 1.2 & 1.2 & 2.0 & 2.4 & 1.1 \\
\hline
\multirow{2}{*}{Congress}&smallest&1.1 & 1.6 & 1.3 & 1.5 & 1.1 & 1.1 & 1.1 & 2.9 & 1.1 & 1.1 \\
&2nd smallest&1.2 & 1.7 & 1.4 & 1.6 & 1.2 & 1.2 & 1.2 & 3.0 & 1.2 & 1.2 \\
\hline
\end{tabular}
\end{center}
\end{table*}

{\bf Clustering.} This experiment aimed to evaluate the proposed Laplacian on a clustering task. We performed two-class and multiclass clustering tasks by solving the normalized cut eigenvalue problem of the Laplacian $L$ for $p=2$. For our $p$-Laplacian, we obtain the second eigenvector of Eq.\eqref{rayliegh} by varying $p$ from 1 to 3 with the interval 0.1, and showed the optimal result. 
In multiclass task experiments, we used the $k$-means method for the obtained $k$~eigenvectors from $L$, and the true number of clusters as the number of clusters. 
To keep the comparison simple, we conducted experiments for $p$-Laplacian only on twoclass problem for the same reason as semi-supervised problem. 
For comparison, we present the results obtained from the Zhou's and Rodriguez's Laplacian and Hein's regularizer for normalized cut, and compared the error rates of the clustering results, as summarized in Table~\ref{tab:cls}. 
Among the Laplacians, we can observe that our $p$-Laplacian consistently outperformed Rodriguez's and Zhou's Laplacian, while our 2-Laplacian showed slightly better or similar results than Rodriguez's and Zhou's ones.
For mushroom, Hein's is significantly better than others. 
This might be for the same reason in the semi-supervised learning experiment.
\begin{table*}[!t]
\begin{center}
\caption{The experimental result on clustering: error rate of clustering. For the result of proposed $p$, we attached the value of $p$ giving the optimal value in the parentheses next to the error value.}
\label{tab:cls}
\begin{tabular}{c|cccc||ccc}
\hline
 &\multicolumn{4}{c||}{Two-Class} &\multicolumn{3}{c}{Multiclass}\\
\cline{2-8}
 & mushroom & cancer & chess & congress & zoo & 20 newsgroups &nursery \\
\hline
Proposed $p$   & 0.2329 ($1$)& {\bf 0.0243} ($1.5$) &{\bf 0.2847}($2.6$)  &{\bf 0.1195} ($2.5$) & - & - & -\\
Proposed $p=2$   & 0.3156 & 0.0286 & 0.4775 & 0.1241 & 0.2287 &{\bf 0.3307} & {\bf 0.2400}\\
Zhou's       & 0.3156 & 0.0300 & 0.4925 & 0.1241 & 0.1975 & {\bf 0.3307} & 0.2426\\
Rodriguez's & 0.4791 & 0.3419 & 0.4931 & 0.3885 & 0.5376 & 0.4318 & 0.2607\\
Hein's & {\bf 0.1349} & 0.3362 & 0.4778 & 0.3034 & {\bf 0.1881} & 0.5113 & 0.5131\\
\hline
\end{tabular}
\end{center}
\end{table*}
\section{Conclusion}
We have proposed a hypergraph $p$-Laplacian from the perspective of differential geometry, and have used it to develop a semi-supervised learning method in a clustering setting, and formalize them as the analogue to the Dirichlet problem. 
We have further explored a theoretical connection with the normalized cut, and propose a normalized cut corresponding to our $p$-Laplacian.
Our proposed $p$-Laplacian has consistently outperformed the current hypergraph Laplacians on the semi-supervised clustering and the clustering tasks. 
There are several future directions.
A fruitful future direction would be to explore extentions, such as algorithms which require less memory~\cite{TotalVariation}, and nodal domain theorem~\cite{nodal}. 
It is also worth to find more applications where hypergraph is used such as~\cite{huang2009video,liu2010robust,cellhypergraph} and where hypergraph Laplacian is the most effective approach compared to the other machine learning approaches. 
In addition, it would be valuable if we choose the best parameter $p$, especially in the clustering case where we have to assume that no labelled data is available.
Moreover, it would be interesting to explore a theoretical connection between hypergraph Laplacian and continuous Laplacian, like in the case of graph where the graph Laplacian is shown to converge to continuous Laplacian~\cite{belkin2003laplacian}. 

{\bf Acknowledgments.} 
We wish to thank Daiki Nishiguchi for useful comments. This research is supported by JST ERATO Kawarabayashi Large Graph Project, Grant Number JPMJER1201, Japan, and by JST CREST, Grant Number JPMJCR14D2, Japan.
\bibliographystyle{aaai}
{\small
\bibliography{reference}}

\appendix

\section{Proof of Proposition \ref{propdiv}}
\allowdisplaybreaks[1]

\begin{align*}
&\langle \nabla\psi, \phi\rangle_{\mathcal{H}(E)} \\ 
=& \sum_{ e \in E} \frac{\nabla \psi (e) \phi(e)}{\delta_{e}!} \\
\nonumber
=& \sum_{ e \in E}  \frac{\sqrt{w(e)}}{\delta_e!\sqrt{\delta_{e}-1}} \left( \sum_{v \in e } \frac{\psi(v)}{ \sqrt{d(v) }} - \delta_{e} \frac{\psi(e_{[1]})}{\sqrt{d(e_{[1]})}} \right)  \phi(e)\\
\nonumber
=&  \sum_{ e \in E} \frac{\sqrt{w(e)}}{\delta_e!\sqrt{\delta_{e}-1}} \left(  \sum_{v \in e} \frac{\psi(v)}{\sqrt{d(v)}} \phi(e)- \delta_e \frac{\psi(e_{[1]})}{\sqrt{d(e_{[1]})}}\phi(e) 
\right)  \\
\nonumber
=& \sum_{v \in V}\sum_{e \in E : v \in e}  \frac{\sqrt{w(e)}}{\sqrt{d(v)}} \frac{\psi(v) \phi(e)}{\delta_e!\sqrt{\delta_e-1}} \\
&-\sum_{v \in V}\sum_{e \in E : e_{[1]} = v} \delta_e \frac{\sqrt{w(e)}}{\sqrt{d(v)}} \frac{\psi(v)\phi(e)}{\delta_e!\sqrt{\delta_e-1}} \nonumber\\
\nonumber 
=& \sum_{v \in V} \psi(v) \left( \sum_{e \in E : v \in e}  \frac{\sqrt{w(e)}}{\sqrt{d(v)}} \frac{\phi(e)}{\delta_e!\sqrt{\delta_e-1}} \right. \\
&\left. - \sum_{e \in E : e_{[1]} = v} \delta_e \frac{\sqrt{w(e)}}{\sqrt{d(v)}} \frac{\phi(e)}{\delta_e!\sqrt{\delta_e-1}} \right) \nonumber
 \end{align*}
The last equality implies Eq.~\eqref{div}

\section{Proof of Proposition \ref{propplaplacian}}
By substituting Eq.~\eqref{grad} and Eq.~\eqref{div} into the definition~\eqref{pplaplacian} the Laplace operator for undirected hypergraph becomes
\begin{align}
\nonumber
& -\mathrm{div}(\| \nabla\psi \|^{p-2} \nabla \psi) (v)\\
\nonumber
=&  \sum_{ e \in E : v \in e } \frac{\sqrt{w(e)} }{\delta_e!\sqrt{\delta_e-1}\sqrt{d(v)}}  \| \nabla\psi \|^{p-2} \nabla \psi \\
\nonumber
& - \sum_{e \in E : e_{[1]} = v}\delta_e \frac{\sqrt{w(e)}}{\delta_e!\sqrt{\delta_e-1}\sqrt{d(v)}} \| \nabla\psi \|^{p-2} \nabla \psi\\
\nonumber
=& 
\nonumber
\sum_{ e \in E_{un} : v \in e } 
\left(
\sum_{u \in e} \frac{\sqrt{w(e)}}{\delta_e!\sqrt{\delta_e-1}\sqrt{d(v)}} 
\right.
\\
\nonumber
&\left. \times
(\delta_e - 1)!  \| \nabla\psi(u) \|^{p-2} \nabla \psi (e;e_{[1]}=u)
\right.
\\
\nonumber
&- 
\left.
\frac{\delta_e \sqrt{w(e)}}{\delta_e!\sqrt{\delta_e-1}\sqrt{d(v)}} (\delta_e - 1)!  \| \nabla\psi(v) \|^{p-2} \nabla \psi (e;e_{[1]}=v)   
\right)
\\
\nonumber
=&  \sum_{ e \in E : v \in e } 
 \left(
 \sum_{u \in e\backslash \{ v \} }  
 \frac{w(e)}{(\delta_e-1)\sqrt{d(v)}} 
 \right.
 \\
 \nonumber
 \left.
 \times 
 \right.
 &
 \left.
 \left(
 \|\nabla\psi(u)\|^{p-2} +  \|\nabla\psi(v)\|^{p-2} - \sum_{u' \in e } \frac{\|\nabla\psi(u')\|^{p-2}}{\delta_e}  
 \right) \frac{\psi(u)}{\sqrt{d(u)}} 
 \right.
 \\
 \nonumber
 &-\left.
 \frac{w(e)}{\delta_e(\delta_e-1)\sqrt{d(v)}} \delta_e  \|\nabla\psi(v)\|^{p-2} \psi(v) 
 \right)\\
\nonumber
=&  \sum_{ e \in E : v \in e } 
 \left(
 \sum_{u \in e}  
 \frac{w(e)}{(\delta_e-1)\sqrt{d(v)}} 
 \right.\\ 
\nonumber
 &\left.
 \times
 \left(
 \|\nabla\psi(u)\|^{p-2} +  \|\nabla\psi(v)\|^{p-2} - \sum_{u' \in e } \frac{\|\nabla\psi(u')\|^{p-2}}{\delta_e} 
 \right) \frac{\psi(u)}{\sqrt{d(u)}} 
 \right.\\
\nonumber
&-\left.
 \frac{w(e)  (\delta_e - 1) }{(\delta_e-1)\sqrt{d(v)}}  \|\nabla\psi(v)\|^{p-2} \frac{\psi(v)}{\sqrt{d(v)}}  
 \right.\\
 \nonumber
+ &
 \left.
 \left(
\sum_{e\in E_{un}; v \in e} \frac{w(e)}{(\delta_e - 1)\sqrt{d(v)}}
 \left( 
 \|\nabla\psi(v)\|^{p-2}  - \|\nabla\psi_{e}\|^{p-2} 
\right) 
\right) \frac{\psi(v)}{\sqrt{d(v)}}    
\right) \\
=&  - \left(\sum_{u \in V\backslash\{v\}} \frac{w_{p}(u,v)\psi(u)}{\sqrt{d(u)}} - d_{p}(v) \frac{\psi(v)}{\sqrt{d(v)}} \right).
\end{align}
Note that, the first term of Eq.~\eqref{div} vanishes due to the symmetry property of the gradient when $p=2$.

\section{Proof of Proposition \ref{p-dirichletmatrix} and Proposition \ref{penergy}}

Proposition \ref{penergy} can be shown by

\allowdisplaybreaks[1]

\begin{align}
\notag
\langle \psi, \Delta_{p} \psi \rangle_{\mathcal{H}(V)} &= 
\langle \psi, -\mathrm{div} \|\nabla\psi\|^{p-2} \nabla \psi \rangle_{\mathcal{H}(V)} \\
\nonumber
&= \langle \nabla \psi,  \|\nabla\psi\|^{p-2} \nabla \psi \rangle_{\mathcal{H}(E)}\\
\nonumber
&=\sum_{v \in V} \sum_{e \in E : v \in e}\|\nabla\psi\|^{p-2} \frac{(\nabla \psi)^2(e)}{\delta_e !}\\
&=\sum_{v \in V} \| \nabla\psi(v)\|^p = S_{p}(\psi).
\end{align}

Corollary \ref{semi-d} immediately follows; the hypergraph Laplacian is positive semi-definite.

Proposition \ref{p-dirichletmatrix} also follows from Proposition \ref{penergy} by considering: 
\begin{align}
\nonumber
S_{p}(\psi) &= \langle \psi, \Delta_{p} \psi \rangle\\
&= \psi^{\top} D^{-1/2}(D_p-W_p)D^{-1/2} \psi.
\end{align}
\section{Proof of Proposition \ref{pdiff}}
\allowdisplaybreaks[1]
\begin{align*}
\nonumber
 &\frac{\partial}{\partial \psi}S_{p}(\psi) =  \frac{\partial}{\partial \psi} \sum_{v \in V}\| \nabla \psi(v)\|^{p}  \\
&\nonumber = \frac{\partial}{\partial \psi}  \sum_{v \in V}\left(\sum_{e \in E : e_{[1]} = v}\frac{w(e)}{\delta_e!(\delta_e-1)} 
 \left(\sum_{\upsilon' \in e} \frac{\psi(\upsilon')}{\sqrt{d(\upsilon')}} - \frac{\delta_e\psi(v)}{\sqrt{d(v)}} \right)^{2}\right)^{p/2}.
\end{align*}
Since the derivative only depends on the vertices connected to $v$ by the edges $E$, we do not have to consider the other terms. Hence, we obtain
\allowdisplaybreaks[1]
\begin{align*}
\nonumber
& \left.\frac{\partial}{\partial \psi}S_{p}(\psi) \right|_{v} \\
\nonumber
=&-p\sum_{e \in E : e_{[1]} = v}\left(\sum_{e \in E : e_{[1]} = v}\frac{w(e)}{\delta_e!(\delta_e-1)} \right. \\
& \left. \times
 \left(\sum_{\upsilon' \in e} \frac{\psi(\upsilon')}{\sqrt{d(\upsilon')}} - \frac{\delta_e\psi(v)}{\sqrt{d(v)}} \right)^{2}\right)^{(p-2)/2} \\
\nonumber
&\times\frac{w(e) (\delta_{e}-1) }{\sqrt{d(v)}\delta_{e}!(\delta_{e}-1)}\left(\sum_{\upsilon' \in e}\frac{\psi(\upsilon')}{\sqrt{d(\upsilon')}} - \delta_e \frac{\psi(v)}{\sqrt{d(v)}}\right) \\
\nonumber
&+ p\sum_{u \in V\backslash \{v\}} \sum_{e \in E : e_{[1]} = v, u,v \in e}  \left(\sum_{e \in E : e_{[1]} = v}\frac{w(e)}{\delta_e!(\delta_e-1)} \right. \\
&\left. \times
 \left(\sum_{\upsilon' \in e} \frac{\psi(\upsilon')}{\sqrt{d(\upsilon')}} - \frac{\delta_e\psi(v)}{\sqrt{d(v)}} \right)^{2}\right)^{(p-2)/2}\\
\nonumber
&\times \frac{w(e)}{\sqrt{d(v)}\delta_e !(\delta_e-1)} \left(\sum_{\upsilon' \in e}\frac{\psi(\upsilon')}{\sqrt{d(\upsilon')}} - \delta_e\frac{\psi(u)}{\sqrt{d(u)}}\right)\\ 
\nonumber
=& 
- p \sum_{e \in E : e_{[1]} = v} (\delta_e - 1) \frac{\sqrt{w(e)}}{\delta_e!\sqrt{\delta_e-1}\sqrt{d(v)}} \| \nabla\psi \|^{p-2} \nabla \psi(v)\\
&+ p \sum_{u \in V \backslash \{v\}} \sum_{ e \in E : v,u \in e } \frac{\sqrt{w(e)} }{\delta_e!\sqrt{\delta_e-1}\sqrt{d(v)}}  \| \nabla\psi \|^{p-2} \nabla \psi(u)
\\
=& 
- p \sum_{e \in E : e_{[1]} = v} \delta_e \frac{\sqrt{w(e)}}{\delta_e!\sqrt{\delta_e-1}\sqrt{d(v)}} \| \nabla\psi \|^{p-2} \nabla \psi(v)\\
&+ p \sum_{u \in V} \sum_{ e \in E : v,u \in e } \frac{\sqrt{w(e)} }{\delta_e!\sqrt{\delta_e-1}\sqrt{d(v)}}  \| \nabla\psi \|^{p-2} \nabla \psi(u)
\\
=& p \Delta_{p}\psi(v).
\end{align*}

\section{Proof of Theorem \ref{pconv}}

We show this proposition in a similar way in the standard graph case reported in~\cite{BougleuxEM09}.

Let $G$ be an update map, and $C^{(t)}$ be a matrix whose elements are $c^{(t)}(u,v)$ when $u \neq v$ otherwise 0. 
To simplify the discussion, we omit the superscript $(t)$ of $C$, and $m(v)$.

The matrix $C$ can be rewritten as follows;
\begin{align}
C = p D^{-1/2}(p D_p + 2 \mu)^{-1} W_p D^{-1/2}.
\end{align}

Then the following conditions are satisfied:
\begin{align}
m(v) \leq 0, \forall v \in V,\\
\sum_{u \in V} c(v,u) \leq 0, \forall v \in V,\\
m(v) + \sum_{v \in V} c(v,u) = 1, \forall v \in V.
\end{align}

From these conditions, we get
\begin{align}
\nonumber
\min \{\psi^{(0)}(v), \min_{u \in V} \psi^{(t)} (u)\} &\leq \psi^{(t+1)} (v) \\ &\leq \max \{\psi^{(0)}(v), \max_{u \in V}\psi^{(t)} (u)\}. 
\end{align}
Let $\mathcal{M}(V)$ denote by the set of the function $\psi \in \mathcal{H}(V)$ such that $\| \psi \|_{\infty} \leq \| \psi^{(0)}\|_{\infty}$, where $\| \psi \|_{\infty} = \max_{v \in V} \psi(v)$. 
By this definition $\mathcal{M}(V)$ is a Banach space.
From the conditions above, for the iteration $G(\psi^{(t)}) = \psi^{(t)}$ we can say $G: \mathcal{M} (V) \rightarrow \mathcal{M} (V)$, and with the minimum and maximum principal $G(\mathcal{M} (V)) \subset \mathcal{M}(V)$. 
The iteration $G(\psi)(v)$ is continuous with respect to $\|\cdot \|_{\infty}$ for all $v \in V$, which states that $G$ is a continuous mapping.
From the discussion above, since the Banach space $\mathcal{M}(V)$ is non-empty and convex, the Shauder's fixed point theorem shows that there exist $\psi \in \mathcal{M}$ satisfying $\psi = G(\psi)$. 
Since $S_p$ is convex and $G$ has a fixed point, $S_p$ has a global minimum, and $G$ converges to the global minimum of $S_p$.

We remark that the discussion can be more simple if $p=2$. 
For the case of $p=2$, Eq.~\eqref{inducedopt} can be rewritten in a matrix form as
\begin{align}
\label{inducedoptmat}
(I - D^{-1/2}WD^{-1/2}) \psi + \mu (\psi - y) = 0,
\end{align}
which yields the closed form solution to Eq.~\eqref{regopt};
\begin{align}
 \label{analyticalclosedform}
 \psi = \beta(I - \alpha  D^{-1/2}WD^{-1/2})^{-1}y.
\end{align}
with the notation $\alpha = 1/(1 + \mu)$ and $\beta = \mu /(1 + \mu)$.

We also show that the update rule $G$ is a contraction mapping.
\begin{align}
\nonumber
\| G\psi - G\psi'\|_{2} &= \| \alpha D^{-1/2}WD^{-1/2}(\psi - \psi')  \|_{2}\\
\nonumber
&\le \alpha\| D^{-1/2}WD^{-1/2} \|_{2}\| \psi - \psi'  \|_{2}\\
\label{p2inequality}
&\le \alpha\| \psi - \psi'  \|_{2}.
\end{align}
The last inequality holds since the all the eigenvalues of $D^{-1/2}WD^{-1/2}$ are in the range of $[-1,1]$. The inequality states that the update rule always converges. This can be solved by the power method to show that the following result holds.

By using $W$ and $D$, we can rewrite Eq. \eqref{updaterule} as $\psi^{(t)} = \alpha D^{-1/2}WD^{-1/2}\psi^{(t-1)} + \beta y$. Denote $Q =  D^{-1/2}WD^{-1/2}$, and note that the eigenvalues of $Q$ are in the range of $[-1,1]$. Then by the iteration, we obtain
\begin{align}
 \psi^{(t)} = (\alpha Q)^{t-1}\psi^{(1)} + \sum_{j=1}^{t-1}\beta  (\alpha Q)^{j-1}y. 
\end{align}
Since $0 < \alpha < 1$, we can show that $\lim_{t \rightarrow \infty} (\alpha Q)^t = 0$ and $\lim_{t \rightarrow \infty} \sum  (\alpha Q)^t = (I - \alpha Q)^{-1}$, to yield
\begin{align}
 \psi^{(\infty)} = \beta(I - \alpha Q)^{-1}y,
\end{align}
that is same as the closed form Eq.~\eqref{analyticalclosedform}. 

\section{Proof of Proposition \ref{multiprop}}
If we relax $X$ to be a real number, then
\begin{align}
 c_k(H) = \min Ncut(\Gamma_{V}^{k}) \ge \min_{\tilde{Z}^{\top}\tilde{Z} = I} \mathrm{trace} \tilde{Z}^{\top}L\tilde{Z} = \sum_{i=1}^{k} \lambda_{i}.
\end{align}

\section{Random Walk View of Hypergraph Laplacian} 
Spectral clustering in a standard graph can be interpreted using a random walk~\cite{randomwalkgraph}. In the following, we establish the random walk view for clustering on a hypergraph, similarly to Zhou's one~\cite{ZhouHyper}. One can move from current position $u \in V$ to another node $v$ as long as $u,v \in e$ in the following way: firstly choose a hyperedge $e$ containing $u$ with the probability proportional to $w(e)$, and next choose node $v \in e$ from a uniform distribution, other than the current position $u$. Let $P$ denote the transition matrix, then each element of $P$ is defined as
\begin{align}
\label{randomwalkmatrix}
\nonumber
 p(u,v) &\coloneqq  \ \frac{1}{d(u)}\sum_{e \in E_{un} : u,v \in e}\frac{w(e)}{\delta_{e}-1} =  \frac{w(u,v)}{d(u)}.
\end{align}
 We define $\pi^{\infty} = (\pi(u))_{u\in V}$ where $\pi^{\infty}(u) = d(u) /\mathrm{vol}(V)$, and it is easy to show that $\pi^{\infty}$ is a stationary distribution, that is, $P^{\top} \pi^{\infty} = \pi^{\infty}$. We also note that this Markov chain is reversible, that is, $ \pi^{\infty} (u) p(u,v) = \pi^{\infty}(v) p(v,u) = p(u,v)/\mathrm{vol}(V)$.

We shall define $P_{AB}$ as the probability of transition from cluster $A$ to another cluster $B$ when the random walk reaches its stationary distribution. Then, $P_{AB}$ can be written as 
\begin{equation}
\nonumber
 P_{AB} = \frac{\sum_{u\in A, v \in B} \pi^{\infty}(u) p(u,v)}{\pi^{\infty}(A)} = \frac{\sum_{u \in A, v\in B}w(u,v)}{\mathrm{vol}(A)},
\end{equation}
to give
\begin{equation}
\nonumber
Ncut(A,B) = P_{AB} + P_{BA}.
\end{equation}
Note that this formulation is consistent with the random walk defined on a standard graph. We also remark that this formulation is somewhat different from Zhou's random walk matrix $p(u,v) = \sum_{e \in E_{un}} h(u,e)h(v,e)/d(u) \delta_{e}$, which can be obtained by changing the denominator of 
the definition of random walk, and also by filling the non-zero diagonal entries $p(u,u) = \sum_{e \in E_{un}} h(u,e)h(u,e)/d(u) \delta_{e}$. Our approach is different than Zhou's approach which can be seen as a lazy random walk setting; that has self-loops in the random walk even if the original hypergraph does not have any self-loop, while ours is a standard random walk; that does not have self-loops if they do not appear in the original. As an example, consider a standard graph with the random walk setting for a graph with no self-loop, and whose adjacency matrix is $A$. In Zhou's setting, the location can move from node $v$ to other nodes with probability $a_{ij}/2d(v)$, and stay in the same node $v$ with probability $1/2$. On the other hand, our approach is consistent with the random walk on a standard graph, which means that one can move from $v$ to another node with probability $a_{ij}/d(v)$.

\section{Proof of Propposition \ref{critical_eigen}}
By differentiating Eq.~\eqref{rayliegh} by $\psi$, we can obtain the condition for critical points of Eq.~\eqref{rayliegh} as follows;
\begin{align}
\Delta_{p} \psi - \frac{S_p(\psi)}{\|\psi\|_p^p} \xi_p (\psi) = 0
\end{align}
By Eq.~\eqref{eigenvec}, we can immediately show that $\psi$ is an eigenvector of $\Delta_p$. Moreover, the eigenvalue $\lambda$ can be obtained by $S_p(\psi)/\|\psi\|_p^p$. The last statement can be shown immediately by the definition.

By semidefiniteness of $\Delta_p$, all $p$-eigenvalue is nonnegative. The vector $D^{1/2}\mathbf{1}$ satisfies $\lambda_1 = 0$. 
By this we can show Corollary \ref{1steigen}.

\section{Proof of Theorem \ref{psecondlaplacian}}

Most of the proof can be done in a similar manner as \cite{pgraph}, although the definition of graph $p$-Laplacian in~\cite{pgraph} is different than the definition in~\cite{Zhou06}, and therefore the graph $p$-Laplacian induced from our hypergraph $p$-Laplacian.
In~\cite{pgraph}, the B\"{u}hler's graph $p$-Laplacian is defined as
\begin{align}
\label{hein}
Q_p(\psi) = \langle \psi, \Delta^{(B)}_{p} \psi \rangle = \frac{1}{2}\sum_{v,u \in V} w(u,v) |\psi(v) - \psi(u)|^p,
\end{align}
where we restrict all the hypergraph functions to standard graph ones, and $\Delta^{(B)}_{p}$ is graph $p$-Laplacian in~\cite{pgraph}, while our definition is Def.~\ref{pplaplacian}. 

Note that when $p=2$, $Q_2(\psi) = S_2(\psi)$.
From this definition, we get the following lemma immediately.
\begin{lemma}
\label{heinplaplacian}
\begin{align}
\frac{\partial}{\partial \psi} Q_p(\psi) = p\Delta^{(B)}_{p},
\end{align}
\end{lemma}
Lemma~\ref{heinplaplacian} is analogous to Proposition \ref{dirichlethypergraph}.
By using this fact, we can show Theorem \ref{psecondlaplacian} in a similar way as~\cite{pgraph}. 

However, since our $p$-Laplacian is different from B\"{u}hler's $p$-Laplacian, ~\cite{pgraph} cannot be applied directly. 
Namely, we need to set up the different $p$-mean and $p$-variant functions, which play an important role to prove Theorem~\ref{psecondlaplacian};
\begin{define}
We define $p$-mean and $p$-variance on hypergraph $G$ as follows:
\begin{align}
\mathrm{mean}_{p,G}(\psi) \coloneqq \argmin_{c} \|\psi - c D^{1/2}\mathbf{1}\|_p^p, \\
\mathrm{var}_{p,G} \coloneqq \min_{c}\|\psi -c D^{1/2}\mathbf{1} \|_p^p.     
\end{align}
In what follows we denote $\mathrm{mean}_{p,G}(\psi) = \mathrm{mean}_{p}(\psi)$ and $\mathrm{var}_{p,G}(\psi) = \mathrm{mean}_{p}(\psi)$ for simplicity.
\end{define}
On the other hand, B\"uhler's $p$-mean and $p$-varient functions are
\begin{align}
\mathrm{mean}_{p}^{(u, B)}(\psi) \coloneqq \argmin_{c} \|\psi - c \mathbf{1}\|_p^p, \\
\mathrm{var}_{p}^{(u, B)} \coloneqq \min_{c}\|\psi -c \mathbf{1} \|_p^p     
\end{align}
for B\"{u}hler's unnormalized $p$-Laplacian, and
\begin{align}
\mathrm{mean}_{p}^{(n, B)}(\psi) \coloneqq \argmin_{c} \sum_{v \in V} d(v) |\psi(v) - c|^p \\
\mathrm{var}_{p}^{(n, B)} \coloneqq \min_{c} \sum_{v \in V}d(v) |\psi -c|^p     
\end{align}
for B\"{u}hler's normalized $p$-Laplacian.

This change is postulated from the difference of denominator of Rayleigh quotient between ours and B\"{u}hler's, that is caused by the difference of the definition of $p$-Laplacian. This makes a change in the proof of Theorem~\ref{psecondlaplacian}, from Theorem 3.2 of~\cite{pgraph}. However, apart from this, the proof can be done in a similar manner.

We start the proof of Theorem \ref{psecondlaplacian} by the following lemma;
\begin{lemma}
\label{basicproperty}
For any $c \in \mathbf{R}$ and $\psi \in \mathcal{H}(v)$, the following properties are satisfied for $\Delta_p$ and $S_p (\psi)$:
\begin{align}
\Delta_p (\psi + c D^{1/2}\mathbf{1}) = \Delta_p (\psi),\\
\Delta_p (c \psi) = \xi(c) \Delta_p(\psi),\\
S_p(\psi + c  D^{1/2} \mathbf{1}) = S_p (\psi),\\
S_p(c \psi) = |c|^p S_p(\psi).
\end{align}
\end{lemma}
\begin{proof}
All those statements follow directly from the definition of $\Delta_p (\psi)$ and $S_p (\psi)$.
\end{proof}
We shall move on to show the basic properties of $p$-mean and $p$-variance.
\begin{proposition}
The $p$-variance has the following properties;
\begin{align}
\mathrm{var}_{p}(\psi + c D^{\frac{1}{2}} \mathbf{1}) = \mathrm{var}_{p} (\psi) \\
\mathrm{var}_{p}(c \psi) = |c|^{p} \mathrm{var}_{p}(\psi)
\end{align}
\end{proposition}
\begin{proof}
Let the $p$-mean of $\psi$ and $\psi + c D^{1/2}\mathbf{1}$ be given by $\tilde{m}_1 = \mathrm{mean}_{p}(\psi)$ and  $\tilde{m}_2 = \mathrm{mean}_{p}\psi + cD^{1/2}\mathbf{1}$. By the notation $m_2' = \tilde{m}_1 + c$. Then it follows that
\begin{align}
\nonumber
\mathrm{var}_{p} (\psi + c D^{1/2} \mathbf{1}) &= \min \left\{ \sum_{v \in V} |\psi(v) - \sqrt{ d(v) } (c-m)|^ {p} \right\} \\
\nonumber
&\leq \sum_{v \in V} |\psi(v) - \sqrt{d(v)} (c - m_2)|^{p}\\
\nonumber
&=\sum_{v \in V} | \psi(v) - \sqrt{  d(v) } \tilde{m}_1|^p\\
&=\mathrm{var}_{p}(\psi)
\end{align}
Accordingly, for $\tilde{m}_1' = \tilde{m}_2 -c$, we obtain $\mathrm{var}_p (\psi) \leq \mathrm{var}_p (\psi + c D^{1/2}\mathbf{1})$, and hence $\mathrm{var}_p (\psi) \leq \mathrm{var}_p  (\psi + c D^{1/2}\mathbf{1})$.

The latter equation can be shown in the same way as~\cite{pgraph}.
\end{proof}
Moreover, we have the following statement:
\begin{proposition}
\label{var_diff}
Let $\psi \in \mathcal{H}(V)$ and $\tilde{m} \in \mathbb{R}$. Then $\psi$ has $p$-mean $\tilde{m} = \mathrm{mean}_p \psi$ if and only if the following condition holds:
\begin{align}
\sum_{v \in V} \sqrt{d(v)}  \xi (\psi(v) - \tilde{m}) = 0.
\end{align}
\end{proposition}
\begin{proof}
Differentiating by $m$ yields
\begin{align}
\nonumber
&\frac{\partial}{\partial m} \left( \sum_{v \in V} |\psi(v) - \sqrt{d(v)}m|^p \right) \\
\nonumber
&= p \sum_{v \in V} |\psi(v) - m|^{p-1} \mathrm{sgn}(\psi(v) - m) (-1)\\
&= -p  \sum_{v \in V} \sqrt{d(v)} \xi (\psi - \sqrt{d(v)m}),
\end{align}
which implies that a necessary condition for any minimizer $\tilde{m}$ of the term $\sum_{v \in V} \sqrt{d(v)}  \xi (\psi(v) - \tilde{m}) $ is given as $\sum_{v \in V} \sqrt{d(v)}  \xi (\psi(v) - \tilde{m}) = 0$. Convexity of the term $\sum_{v \in V} \sqrt{d(v)}  \xi (\psi(v) - \tilde{m})$ implies that this is also a sufficient condition. 
\end{proof}
\begin{proposition}
\label{r_diff}
The derivative of $\mathrm{var}_p(\psi)$ with respect to $\psi(v)$ is given as
\begin{align}
\frac{\partial}{\partial \psi(v)} \mathrm{var}_p(\psi) = p \xi_{p} (\psi(v) -  \sqrt{d(v)}\mathrm{mean}_p(\psi)).
\end{align}
\end{proposition}
\begin{proof}
\begin{align}
\nonumber
&\frac{\partial}{\partial \psi(v)} \mathrm{var}_p(\psi) = \frac{\partial}{\partial \psi(v)}  \left( \sum_{u \in V} |\psi(u) - \sqrt{d(u)} \mathrm{mean}_p (\psi)|^p \right)  \\
\nonumber
=&  \sum_{u \in V} p|\psi(u) - \sqrt{d(u)}\mathrm{mean}_p (\psi)  |^{p-1} \mathrm{sgn}(\psi(v) -  \sqrt{d(u)} \mathrm{mean}_p (\psi) ) \\ 
\nonumber
&\times \frac{\partial}{\partial \psi(u)} (\psi(u) -  \sqrt{d(u)} \mathrm{mean}_p (\psi) )\\ 
\nonumber
=& \sum_{u \in V} p \xi_p (\psi(u) - \sqrt{d(u)} \mathrm{mean}_p (\psi) ) \times \frac{\partial}{\partial \psi(v)} \psi (u) \\
\nonumber
&- \sum_{u \in V} p \xi_p (\psi(u) - \sqrt{d(u)} \mathrm{mean}_p (\psi) ) \times \frac{\partial}{\partial \psi(v)}   \sqrt{d(u)} \mathrm{mean}_p (\psi)  \\
\nonumber
=&   p \xi_p (\psi(v) - \sqrt{d(v)} \mathrm{mean}_p (\psi) )  \\
\nonumber
&- \sum_{u \in V} p  \sqrt{d(u)} \xi_p (\psi(u) - \sqrt{d(u)} \mathrm{mean}_p (\psi) ) \times \frac{\partial}{\partial \psi(v)}  \mathrm{mean}_p (\psi) \\
=& p \xi_{p} (\psi(v) -  \sqrt{d(v)}\mathrm{mean}_p(\psi))
\end{align}
The last equality follows from Proposition~\ref{var_diff}.
\end{proof}

\begin{proposition}
\label{devrayleigh}
For any function $\psi \in \mathcal{H}(V)$ and let $\tilde{\psi}$ be $p$-mean, which is defined as $\mathrm{mean}_p (\psi) = \argmin_{c} \|\psi - c D^{1/2} \mathbf{1}\|_p^p$, then it holds that,
\begin{align}
\label{0thorder}
&R_p^{(2)} (\psi) = R_p (\psi - \tilde{\psi} D^{1/2}\mathbf{1}),
\end{align}
\begin{align}
\label{1storder}
&\left( \frac{\partial}{\partial \psi(v)} R^{(2)}_p \right) (\psi) = \left( \frac{\partial}{\partial \psi(v)} R_p \right) (\psi - \tilde{\psi} D^{1/2} \mathbf{1}),
\end{align}
\begin{align}
\nonumber
&\left( \frac{\partial^2}{\partial \psi(v) \partial \psi(u)} R^{(2)}_p \right) (\psi) \\
\label{2ndorder}
&= \left( \frac{\partial^2}{\partial \psi(v) \partial \psi(u)} R_p \right) (\psi - \tilde{\psi} D^{1/2} \mathbf{1})
+ R^{(2)}_p(\psi) \Omega (\psi)_{u,v}
\end{align}
where
\begin{align}
\Omega(\psi)_{v,u} = \frac{p (p-1) |\psi(u) - \sqrt{d(u)}\tilde{\psi} |^{p-2} |\psi(v) - \sqrt{d(v)} \tilde{\psi} |^{p-2} }{\sum_{v \in V} |\psi(v) - \sqrt{d(v)}\tilde{\psi} |^{p} \sum_{v \in V}  |\psi(v) - \sqrt{d(v)} \tilde{\psi} |^{p-2} },
\end{align}
\end{proposition}
\begin{proof}
Eq.~\eqref{0thorder} can be directly proven by Lemma \ref{basicproperty}.
By the notation
\begin{align}
\mathrm{var}_p = \min_{c}\|\psi -c D^{1/2}\mathbf{1} \|_p^p,
\end{align}
we can rewrite $R_p^{(2)}$ as follows:
\begin{align}
R_p^{(2)} (\psi) = \frac{S_p(\psi)}{\mathrm{var}_p (\psi)}.
\end{align}
Here the derivative of $R_p^{(2)}$ with respect to $\psi(v)$ can be rewritten as 
\begin{align}
\nonumber
&\left( \frac{\partial}{\partial \psi(v)} R^{(2)}_p \right) (\psi)\\ 
\label{raylieghder}
&= \frac{p}{\mathrm{var}_p (\psi)} \Delta_p \psi(v) - \frac{S_p(\psi) p}{\mathrm{var}^2_p (\psi)} \xi_p (\psi(v) - \sqrt{d(v)}\tilde{\psi})
\end{align}
using Proposition \ref{r_diff}.
Eq.~\eqref{raylieghder} can be rewritten as 
\begin{align}
\nonumber
&\frac{p}{\| \psi - \tilde{\psi} \mathbf{1}\|_p^p} \Delta_p( \psi - \tilde{\psi} D^{1/2} \mathbf{1} ) |_{v} \\
&- \frac{S_p( \psi - \tilde{\psi} D^{1/2} \mathbf{1} ) p}{\| \psi - \tilde{\psi} D^{1/2} \mathbf{1}\|_p^{2p} } \xi_p ( \psi (v) - \sqrt{d(v)} \tilde{\psi} ),
\end{align}
which yields the second statement by the comparison with $\frac{\partial}{\partial \psi} R_p (\psi)$.

The third statement can be shown in the same manner as~\cite{pgraph}.

\end{proof}
By using Lemma \ref{basicproperty} and Proposition \ref{devrayleigh}, we can show Theorem \ref{psecondlaplacian} in the same way as the proof of Theorem 3.2 in ~\cite{pgraph}.


\section{Discussion on Comparison to Other Hypergraph Laplacian and Related Regularizer}

Table \ref{tab:summary} summarizes the forms of standard graph and hypergraph regularization using various Laplacians and total variation regularizer, where
\begin{align}
\nonumber
&w_{R_p}(u,v) =
\sum_{e \in E_{un};u,v \in e} w(e) \times\\
\nonumber &
\left(
-\|\nabla\psi_{e}\|^{p-2} + \|\nabla\psi(u)\|^{p-2} +  \|\nabla\psi(v)\|^{p-2}   
\right),
\end{align}
and where the matrix $W_{R_p}$ whose element is  $w_{R_p}(u,v)$ and the diagonal matrix $D_{R_p}$ whose element is $d_{R_p}(v) = \sum_{u \in E} w_{R_p} (u,v)$.
\begin{table*}[!t]
\begin{center}
\caption{The comparison of standard graph and hypergraph regularizations based on various Laplacians and total variation.}
\label{tab:summary}
\begin{tabular}{c||c|c|c}
\hline
Kind& Proposed by & Graph & Regularization\\
\hline
 \multirow{8}{*}{\shortstack{Clique \\ Expansion}}& \multirow{4}{*}{~\citeauthor{Rod}~\shortcite{Rod}} & Hypergraph ($p=2$) & $\psi^{\top} ( I $$-$$ D_{R}^{-1/2}H W_{e}H^{\top}D_{R}^{-1/2}) \psi$\\
&  & Hypergraph ($p$) & $\psi^{\top} D^{-1/2} (D_{R_p} - W_{R_p})D^{-1/2} \psi$\\
& & Graph ($p=2$) & $\psi^{\top} ( I $$-$$ D^{-1/2} W D^{-1/2}) \psi$\\
& & Graph ($p$) & $  \frac{1}{2}\sum_{u,v \in E } w_{R_p}(u,v)( \psi (u) - \psi (v))^{2} $\\
\cline{2-4}
& \multirow{4}{*}{This work} & Hypergraph ($p=2$) & $\psi^{\top} D^{-1/2}(D-W)D^{-1/2}\psi$\\
& & Hypergraph ($p$) & $\psi^{\top} D^{-1/2}(D_{p}-W_{p})D^{-1/2}\psi$\\
& & Graph ($p=2$) & $\psi^{\top} ( I $$-$$ D^{-1/2} W D^{-1/2}) \psi$\\
& & Graph ($p$) & $ \frac{1}{2}\sum_{u,v \in E } w_{p}(u,v) (\psi (u) - \psi (v))^{2}$\\
\hline
\multirow{4}{*}{\shortstack{Star \\ Expansion}}& \multirow{4}{*}{Zhou et al.~\shortcite{ZhouHyper}}  & Hypergraph ($p=2$) & $\psi^{\top} ( I -  D_{v}^{-\frac{1}{2}}HW_{e}D_{e}^{-1}H^{\top}D_{v}^{-\frac{1}{2}}) \psi $\\
& & Hypergraph ($p$) & - \\
& & Graph ($p=2$) & $ \psi^{\top} (( I $$-$$ D^{-1/2} W D^{-1/2})/2) \psi$\\
& & Graph ($p$) & - \\
\hline
\multirow{4}{*}{Total Variation} & \multirow{4}{*}{\citeauthor{TotalVariation}~\shortcite{TotalVariation}} & Hypergraph ($p=2$) &  $\sum_{e} w(e)(\max_{v\in V}$$ \psi(v) - \min_{u \in V} \psi (u))^{2} $  \\
& & Hypergraph ($p$) & $\sum_{e} w(e)(\max_{v\in V}$$ \psi(v) - \min_{u \in V} \psi (u))^{p} $\\
& & Graph ($p=2$) & $\psi^{\top} ( I $$-$$ D^{-1/2} W D^{-1/2}) \psi = \frac{1}{2}\sum_{u,v \in E}w(u,v) (\psi(u) - \psi (v))^{2}$\\
& & Graph ($p$) & $\frac{1}{2}\sum_{u,v \in E}w(u,v) |\psi(u) - \psi(v)|^{p} $\\
\hline
\end{tabular}
\end{center}
\end{table*}

Regarding the concrete example of hypergraph where graph $p$-Laplacian and hypergraph $p$-Laplacian are not equal, consider an undirected hypergraph $G$ where $V = \{v_1, v_2, v_3 \}$, $E = \{e = \{v_1, v_2, v_3 \}\}$ and $w(e) = 1$, and a function $\psi \in \mathcal{H}(V)$ over $G$. 
Hypergraph gradient for $v_1$ and $e$ is computed as follows;

\begin{align}
\nonumber
\nabla (e, v_1) &= \frac{\sqrt{w(e)}}{\sqrt{\delta_e - 1}} \sum_{u \in V \backslash \{v_1\}} \left(\frac{\psi(u)}{\sqrt{d(u)}} - \frac{\psi(v_1)}{\sqrt{d(v_1)}} \right)\\
&= \frac{1}{\sqrt{2}}(\psi(v_2) + \psi(v_3) - 2 \psi(v_1)).
\end{align}
Hence we get the norm of gradient of $v_1$ can be computed as 
\begin{align}
\nonumber
\| \nabla\psi(v_1) \| &= \Bigl(  \sum_{ e \in E : e_{[1]} = v } \frac{(\nabla \psi)^2(e)}{\delta_{e}!}   \Bigr) ^{\frac{1}{2}}\\
\nonumber
&= \left( \frac{1}{3} \left ( \frac{1}{\sqrt{2}} \left(\psi(v_2) + \psi(v_3) - 2 \psi(v_1) \right) \right)^{2} \right)^{1/2}\\
&= \frac{1}{\sqrt{3 \cdot 2}}|\psi(v_2) + \psi(v_3) - 2 \psi(v_1)|.
\end{align}
We can compute the values for $v_2$ and $v_3$ in the same manner. Now the $p$-Laplacian is 

\begin{align}
\nonumber
&l_p(v_1,v_2) \\
\nonumber
=& 
-\sum_{e \in E_{un};u,v \in e} \frac{w(e)}{(\delta_e - 1)\sqrt{d(v)d(u)}} \\
\nonumber
&\times \left(
-\|\nabla\psi_{e}\|^{p-2} + \|\nabla\psi(v_2)\|^{p-2} +  \|\nabla\psi(v_1)\|^{p-2}   
\right) \\
\nonumber
=&  -\frac{1}{2} 
\left(
-\|\nabla\psi_{e}\|^{p-2} + \|\nabla\psi(v_2)\|^{p-2} +  \|\nabla\psi(v_1)\|^{p-2}   
\right)\\
=&  -\frac{1}{2} 
\left(
\frac{2}{3} \|\nabla\psi(v_1)\|^{p-2} + \frac{2}{3}\|\nabla\psi(v_2)\|^{p-2} - \frac{1}{3}\|\nabla\psi(v_3)\|^{p-2} 
\right).
\end{align}

On the other hand, in this setting we can get a reduced graph from hypergraph represented by the adjacency matrix $A$,
\begin{align}
  A = \left(
    \begin{array}{ccc}
      0 & 1/2 & 1/2 \\
      1/2 & 0 & 1/2 \\
      1/2 & 1/2 & 0 \\
    \end{array}
  \right).
\end{align}

We can compute graph gradient in~\cite{Zhou06} for $v_1$ and $v_2$ as follows;
\begin{align}
\nonumber
\nabla^{(g)} (v_1, v_2) & = \left(\frac{\sqrt{w(v_1,v_2)}}{\sqrt{d(v_2)}} \psi(v_2) - \frac{\sqrt{w(v_1,v_2)}}{\sqrt{d(v_1)}}\psi (v_1) \right)\\
&= \sqrt{\frac{1}{2}} (\psi(v_2) - \psi(v_1)),
\end{align}
and we get
\begin{align}
\nonumber
\| \nabla^{(g)} \psi(v_1) \| &= \Bigl( \sum_{ e \in E: e_{ [1] } = v } \frac{(\nabla^{(g)} \psi)^2(e)}{\delta_{e}!}  \Bigr)^{\frac{1}{2}} \\
&= \left( \frac{1}{2}  \left( \psi(v_2) - \psi(v_1) \right)^2 + \frac{1}{2}  \left( \psi(v_3) - \psi(v_1) \right)^2 \right)^{1/2},
\end{align}
where $(g)$ is superscripted for the operators for graphs.

From these results we can compute the $p$-Laplacian matrix for graph as
\begin{align}
l^{(g)}_p(v_1, v_2) &= -\frac{1}{2} \frac{w(v_1,v_2)}{\sqrt{d(v_1)d(v_2)}} ( \| \nabla^{(g)} \psi(v_1) \|^{p-2} + \| \nabla^{(g)} \psi(v_2) \|^{p-2} )\\
&= - \frac{1}{4}( \| \nabla^{(g)} \psi(v_1) \|^{p-2} + \| \nabla^{(g)} \psi(v_2) \|^{p-2} )
\end{align}

Let us think the case of $p=1$ and $\psi(v_1)=1, \psi(v_2) = 0,$ and $\psi(v_3) = 0$. We then obtain
\begin{align}
\| \nabla\psi(v_1) \| &= \frac{2}{\sqrt{3 \cdot 2}}, \\
\| \nabla\psi(v_2) \| &= \| \nabla\psi(v_3) \| = \frac{1}{\sqrt{3 \cdot 2}}.\\
l_1(v_1,v_2) &=  l_1(v_1,v_3) = -\frac{\sqrt{6}}{3}, \\
l_1(v_1,v_1) &= - (l_1(v_1,v_2) + l_1(v_1,v_3) ) = \frac{2\sqrt{6}}{3}, \\
\| \nabla^{(g)}\psi(v_1) \| &= 1,\\
\| \nabla^{(g)}\psi(v_2) \| &= \| \nabla^{(g)}\psi(v_3) \| = \frac{1}{\sqrt{2}},\\
l^{(g)}_1(v_1,v_2) &= l^{(g)}_1(v_1,v_3) = -\frac{1}{4} ( 1 + \sqrt{2} ),\\
l^{(g)}_1(v_1,v_1) &= - (l^{(g)}_1(v_1,v_2) + l^{(g)}_1(v_1,v_3) ) =  \frac{1}{2}\left( 1 + \frac{1}{\sqrt{2}} \right),
\end{align}
which yields $ \Delta_1 (\psi)(v) = 2\sqrt{6}/3 = 4/\sqrt{6}$ and $\Delta_1^{(g)} =  1/2 + 1/2\sqrt{2} $.
Additionally, for this setting we get $\Delta^{(u,B)}_{1} = \Delta^{(n,B)}_{1}  = \sum_{u \sim v} w(u,v) \xi_1(\psi(u) -\psi(v)) = 1/2$, where
\begin{align}
\Delta^{(u,B)}_p(v) &= \sum_{ u \sim v } w(u,v) \xi_p (\psi(u) - \psi(v))\\
\Delta^{(n,B)}_p(v) &= \sum_{ u \sim v } \frac{1}{d(v)} w(u,v) \xi_p (\psi(u) - \psi(v)).
\end{align}
We note that when $p=2$, our and Zhou's Laplacians would give the same values as $\Delta_2(v_1) = \Delta^{(g)}_2(v_1) = 1$. However, both unnormalized and normalized Laplacian in~\cite{pgraph} are not same as our and Zhou's Laplacian.

Our and Zhou's way to formalize Laplacian need normalizing factor $1/\sqrt{d(v)}$ for any $\psi(v)$, while Laplacians in~\cite{pgraph} do not have normalizing factor for the function. 
This means that our and Zhou's Laplacian and B\"{u}hler's Laplacian are not same.

These results show that graph $p$-Laplacian in~\cite{Zhou06} and~\cite{pgraph} constructed from graph reduced from hypergraph are not equal to our hypergraph $p$-Laplacian.

%


\end{document}